%% file: arxivpaper.tex
\documentclass[a4paper]{article}

\newcommand{\short}[1]{}

\usepackage{microtype}
\usepackage{graphicx}
\usepackage{subfigure}
\usepackage{booktabs} 
\usepackage{amsfonts,amsmath,amsthm,amssymb}
\usepackage{mathtools}
\usepackage{xcolor}
\usepackage{selectp}
\usepackage{natbib}
\usepackage{algorithm,algorithmic}
\usepackage{fullpage}
\usepackage[T1]{fontenc}
\usepackage[utf8]{inputenc}

\usepackage{hyperref}
\usepackage[capitalise,nameinlink]{cleveref}

\usepackage{enumitem}
\usepackage{nicefrac}
  
\include{header}

\newcommand{\cS}{\mathcal{S}}

\newcommand{\R}{\mathbb{R}}
\newcommand{\calN}{\mathcal{N}}

\newcommand{\calA}{\mathcal{A}}

\newcommand{\cost}{J}
\newcommand{\KSig}{\mathcal{K}}
\newcommand{\SigK}{\mathcal{E}}

\newcommand{\trbound}{\nu}

\DeclareMathOperator{\cov}{Cov}

\newcommand{\frob}{\mathsf{F}}

\theoremstyle{plain}

\begin{document}

\title{Online Linear Quadratic Control}

\author{%
Alon Cohen\textsuperscript{1,2} \qquad 
Avinatan Hassidim\textsuperscript{1,3} \qquad 
Tomer Koren\textsuperscript{4} \qquad
Nevena Lazic\textsuperscript{4} \\
Yishay Mansour\textsuperscript{1,5} \qquad
Kunal Talwar\textsuperscript{4} 
\\[0.25cm]
\normalsize \textsuperscript{1}Google Research \qquad \textsuperscript{2}Technion--Israel Institute of Technology \qquad \textsuperscript{3}Bar-Ilan University \\
\normalsize \textsuperscript{4}Google Brain, Mountain View \qquad \textsuperscript{5}Tel-Aviv University}


\maketitle

\input{paper}

%% file: header.tex
\usepackage{relsize}

\usepackage{bm}

\usepackage[capitalise,nameinlink]{cleveref}
\usepackage{xcolor}
\usepackage{dsfont}

\newcommand{\ignore}[1]{}

\theoremstyle{definition}
\newtheorem{theorem}{Theorem}[section]

\newtheorem{lemma}[theorem]{Lemma}

\newtheorem{observation}[theorem]{Observation}

\newtheorem*{theorem*}{Theorem}
\newtheorem*{lemma*}{Lemma}
\newtheorem*{corollary*}{Corollary}
\newtheorem*{proposition*}{Proposition}
\newtheorem*{claim*}{Claim}
\newtheorem*{fact*}{Fact}
\newtheorem*{observation*}{Observation}

\theoremstyle{definition}
\newtheorem{definition}[theorem]{Definition}
\newtheorem{remark}[theorem]{Remark}

\newtheorem*{definition*}{Definition}
\newtheorem*{remark*}{Remark}
\newtheorem*{example*}{Example}
\newtheorem*{question*}{Question}

 \theoremstyle{plain}
\newtheorem*{theoremaux}{\theoremauxref}
\gdef\theoremauxref{1}

\DeclareMathAlphabet{\mathbfsf}{\encodingdefault}{\sfdefault}{bx}{n}

\DeclareMathOperator*{\trace}{Tr}

\newcommand{\lr}[1]{\mathopen{}\left(#1\right)}
\newcommand{\Lr}[1]{\mathopen{}\big(#1\big)}

\newcommand{\lrbra}[1]{\mathopen{}\left[#1\right]}
\newcommand{\Lrbra}[1]{\mathopen{}\big[#1\big]}
\newcommand{\LRbra}[1]{\mathopen{}\Big[#1\Big]}

\newcommand{\norm}[1]{\|#1\|}

\newcommand{\abs}[1]{|#1|}

\newcommand{\wh}[1]{\smash{\widehat{#1}}}
\renewcommand{\O}{O}

\newcommand{\E}{\mathbb{E}}

\newcommand{\tr}{^{\mkern-1.5mu\mathsf{T}}}

\newcommand{\st}{\star}

\newcommand{\reals}{\mathbb{R}}

\newcommand{\thalf}{\tfrac{1}{2}}

\newcommand{\eqdef}{:=}

\renewcommand{\leq}{~\le~}

\let\oldtfrac\tfrac
\renewcommand{\tfrac}[2]{\smash{\oldtfrac{#1}{#2}}}

\let\nablaold\nabla
\renewcommand{\nabla}{\nablaold\mkern-1mu}


%% file: paper.tex

\begin{abstract}
We study the problem of controlling linear time-invariant systems with known noisy dynamics and adversarially chosen quadratic losses.
%
%
We present the first efficient online learning algorithms in this setting that guarantee $O(\sqrt{T})$ regret under mild assumptions, where $T$ is the time horizon.
Our algorithms rely on a novel SDP relaxation for the steady-state distribution of the system.
Crucially, and in contrast to previously proposed relaxations, the feasible solutions of our SDP all correspond to ``strongly stable'' policies that mix exponentially fast to a steady state.
%
%
%
%
%
\end{abstract}

\section{Introduction}

Linear-quadratic (LQ) control is one of the most widely studied
problems in control theory
\citep{anderson1972linear,bertsekas1995dynamic,zhou1996robust}. It
has been applied successfully to problems in statistics,
econometrics, robotics, social science and physics. In recent years,
it has also received much attention from the machine learning
community, as increasingly difficult control problems have led to
demand for data-driven control systems
\citep{abbeel2007application,levine2016end,sheckells2017robust}.

In LQ control, both the state and action are real-valued vectors.
The dynamics of the environment are linear in the state and action, and are perturbed by Gaussian noise.
The cost is quadratic in the state and control (action) vectors.
The optimal control policy, which minimizes the cost, selects the control vector as a linear function of the state vector, and can be derived by solving the algebraic Ricatti equations.
%


%
%

The main focus of this work is control of linear systems whose quadratic costs vary in an unpredictable way.
This problem may arise in settings such as building climate control in the presence of time-varying energy costs, due to energy auctions or unexpected demand fluctuations.
To measure how well a control system adapts to time-varying costs, it is common to consider the notion of regret: the difference between the total cost of the controller, one that is only aware of previously observed costs, and that of the best fixed control policy in hindsight.
%
This notion has been thoroughly studied in the context of online
learning, and particularly in that of online convex optimization
\citep{cesa2006prediction,hazan2016introduction,shalev2012online}.
LQ control was considered in the context of regret by \citet{abbasi2014tracking}, who give a learning algorithm for the problem of tracking an adversarially changing target in a system with noiseless linear dynamics.

In this paper we consider online learning with fixed, known, linear dynamics and adversarially chosen quadratic cost matrices. 
Our main results are two online algorithm that achieve $O(\sqrt{T})$
regret, when comparing to any fast mixing linear policy.\footnote{
Technically, we define the class of ``strongly stable'' policies
that guarantee the desired fast mixing property. Conceptually,
slowly mixing policies are less attractive for implementation, given
their inherent gap between their long and short term cost.} One of
our online algorithms is based on Online Gradient Descent of
\cite{zinkevich2003online}. The other is based on Follow the Lazy
Leader of \cite{kalai2005efficient}, a variant of Follow the
Perturbed Leader with only $O(\sqrt{T})$ expected number of policy
switches.

Overall, our approach follows \citet{even2009online}. We first show
how to perform online learning in an  ``idealized setting'', a
hypothetical setting in which the learner can immediately observe
the steady-state cost of any chosen control policy. We proceed to
bound the gap between the idealized costs and the actual costs.

Our technique is conceptually different to most learning problems:
instead of predicting a policy and observing its steady-state cost,
the learner predicts a steady-state distribution and derives from it
a corresponding policy. Importantly, this view allows us to cast the
idealized problem as a semidefinite program which minimizes the
expected costs as a function of a steady state distribution (of both
states and controls). As the problem is now convex, we apply OGD and
FLL to the SDP and argue about fast-mixing properties of its
feasible solutions.

%

For online gradient descent, we define a ``sequential strong
stability'' property that couples consecutive control matrices, and
show that it guarantees that the observed state distributions
closely track those generated in the idealized setting. We then show
that the sequence of policies generated by the online gradient
descent algorithm satisfies this property.
In Follow the Lazy Leader, following each switch our algorithm
resets the system---a process that takes a constant number of
rounds,
after which the cost of playing the new policy is less than its
steady-state cost.


The holy grail of reinforcement learning is controlling a dynamical stochastic system under uncertainty,
and clearly both MDPs and LQ control are well within this mission statement.
There are obvious differences between the two models: MDPs model discrete state and action dynamics while LQ control addresses continuous linear dynamics with a quadratic  cost.
%
In this work we are inspired by methodologies from online-MDP and regret minimization to derive new results for LQ control.
We believe that exploring the interface between the two will be fruitful for both sides, and holds significant potential for future RL research agenda.


\subsection{Related Work}

LQ control can be seen as a continuous analogue of the discrete
Markov Decision Process (MDP) model. As such, our results are
conceptually similar to those of \citet{even2009online}, who derive
regret bounds for MDPs with known dynamics and changing rewards.
However, our technical approach and the derivation of our algorithms
are very different than those applicable in  context of MDPs.

Among the many follow-up works to \citet{even2009online}, let us
note \citet{yu2009markov} and \citet{abbasi2013online} that propose
lazy algorithms similar to our second algorithm.
We remark that, compared to our $O(\sqrt{T})$ regret bounds, \citet{abbasi2014tracking} give an $O(\log^2 T)$ regret bound under much stronger assumptions.%
\footnote{Not only their setting assumes that $Q_t=Q$ and $R_t=I$
for all~$t$ for a fixed and known matrix $Q \succeq 0$, they also
make non-trivial norm assumptions on the corresponding optimal
control matrix~$K^\st$.} Similar bounds are established by
\citet{neu2017fast} for online learning in linearly solvable MDPs,
that were shown to capture appropriately discretized versions of LQ
control systems \citep{todorov2009efficient}. In light of these
results, it is interesting to investigate whether our bounds are
tight or can actually be improved. We leave this investigation for
future work.

An orthogonal line of research that has gained popularity in recent
years is controlling linear quadratic systems with unknown fixed
dynamics.
The majority of recent papers deal with off-policy learning: either
by policy gradient \citep{fazel2018global}; by estimating the
transition matrices \citep{dean2017sample}; or by improper learning
\citep{hazan2017learning,arora2018towards}. In contrast to that,
\citet{abbasi2011regret} and \citet{ibrahimi2012} present an
on-policy learning algorithm with $O(\sqrt{T})$~regret.

Semidefinite programming for LQ control has been previously used
\citep{balakrishnan2003semidefinite,dvijotham2013convex,lee2016semidefinite},
mostly in the context of infinite-horizon constrained LQRs
\citep{lee2007constrained,schildbach2015linear}. In many of these
formulations, one has to solve the SDP exactly to obtain a
stabilizing solution; in other words, only the optimal policy is
known to be stable and suboptimal policies need not be stabilizing.
This is not the case in our SDP formulation, as any feasible
solution is not only stable but, in fact, strongly-stable (see the
formal definition in \cref{sec:stability}).

\section{Background}

\subsection{Linear Quadratic Control}

The standard linear quadratic (Gaussian) control problem is as follows.
Let $x_t \in \R^d$  be the system state at time $t$ and let $u_t \in \R^k$ be the control (action) taken at time $t$. The system transitions to the next state using linear time-invariant~dynamics
\[
x_{t+1} = A x_t + B u_t + w_t~,
\]
where $w_t$ are i.i.d.~Gaussian noise vectors with zero mean and covariance $W \succeq 0$ . The cost incurred at each time point is a quadratic function of the state and control,
$x_t\tr Q x_t + u_t\tr R u_t$, for positive definite matrices $Q$ and $R$.

A policy is a mapping $\pi : \R^d \mapsto \R^k$ from the current state $x_t$ to a control (i.e., an action) $u_t$.
The cost of a policy after $T$ time steps is
\[
\cost_T(\pi)
=
\E\LRbra{ \sum_{t=1}^T x_t\tr Q x_t + u_t\tr R u_t }
,
\]
where $u_1,\ldots,u_T$ are chosen according to $\pi$; the expectation is w.r.t.~the randomness in the state transitions and (possibly) the policy.
In the infinite-horizon version of the problem, the goal is to minimize the steady-state cost $\cost(\pi) = \lim_{T \to \infty} (1/T) \cost_T(\pi)$.

In the infinite-horizon setting and when the system is controllable,%
\footnote{The system is controllable if the matrix $(B ~ AB ~ \cdots ~ A^{d-1}B)$
has full column-rank. Under the controllability assumption, any state can be reached in at most $d$ steps (ignoring noise).}
it is well-known that the optimal policy is given by constant linear feedback $u_t = Kx_t$.
%
For the optimal $K$, the dynamics are given by $x_{t+1} = (A+BK)x_t + w_t$, and $K$ is guaranteed to be stable; a policy $K$ is called stable if $\rho(A+BK) < 1$, where for a matrix $M$, $\rho(M)$ is the spectral radius of $M$.
In this case, $x_t$ converges to a steady-state (stationary) distribution, i.e., $x_t$ has the same distribution as $(A+BK)x_t + w_t$. This implies that $\E[x_t] = 0$, and the covariance matrix $X = \E[x_t x_t\tr]$ satisfies $X = (A+BK) X (A+BK)\tr + W$.

The steady-state cost of a stable policy $K$ with steady-state covariance $X$ is given by $\cost(K) = (Q + K\tr R K) \bullet X$. Here $\bullet$ denotes element-wise inner product, i.e., $A \bullet B = \trace(A\tr B)$.
%

\subsection{Problem Setting}

We consider an online setting, where a sequence of positive definite cost matrices  $Q_1,\ldots,Q_T$, $R_1,\ldots,R_T$ is chosen by the environment ahead of time and unknown to the learner.
We assume throughout that $\trace(Q_t),\trace(R_t) \le C$ for all $t$, for some constant $C>0$.
We assume that the dynamics $(A, B)$ are time-invariant and known, and that the system is initialized at $x_0=0$. At each time step $t$, the learner observes the state $x_t$, chooses an action $u_t$, and suffers cost $x_t\tr Q_t x_t + u_t\tr R_t u_t$. Thereafter, the system transitions to the next state.


A (randomized) learning algorithm $\calA$ is a mapping from $x_t$ and the previous cost matrices $Q_0, ..., Q_{t-1}$ and $R_0,\ldots,R_{t-1}$ to a distribution over a control $u_t$. We define the cost of an algorithm as $\cost_T(\calA) = \E[\sum_{t=1}^T x_t\tr Q_t x_t + u_t\tr R_t u_t]$, where $u_1,\ldots,u_T$ are chosen at random according to $\calA$.

The goal of the learner is to minimize the regret, defined as:
\begin{align*}
R_T(\calA)
=
\cost_T(\calA) - \min_{\pi \in \Pi} \cost_T(\pi)
~,
\end{align*}
where $\Pi$ is a set of benchmark policies.
In the sequel, we fix $\Pi$ to be the set of all strongly stable policies; we defer the formal definition of this class of policies to \cref{sec:stability} below.





\section{Strong Stability}
\label{sec:stability}

In this section we formalize the notion of a strongly stable policy and discuss some of its properties.
Intuitively, a strongly stable policy is a policy that exhibits fast mixing and converges quickly to a steady-state distribution.
Note that, while stable policies $K$ (for which $\rho(A+BK)<1$) necessarily converge to a steady-state, nothing is guaranteed regarding their rate of convergence. The following definition helps remedy that.

\begin{definition}[Strong Stability]
A policy $K$ is \emph{$(\kappa,\gamma)$-strongly stable} (for $\kappa > 0$ and $0 < \gamma \le 1$) if $\norm{K} \le \kappa$, and there exists matrices $L$ and $H$ such that $A+BK = HLH^{-1}$, with $\norm{L} \le 1-\gamma$ and $\norm{H}\norm{H^{-1}} \le \kappa$.
\end{definition}

Strong-stability is a quantitative version of stability, in the sense that any stable policy is strongly-stable for some $\kappa$ and $\gamma$ (See \cref{lem:lyapunov} in the supplementary material).
Conversely, strong-stability implies stability: if $K$ is strongly-stable then $A+BK$ is similar to a matrix $L$ with $\norm{L} < 1$, and so $\rho(A+BK) = \rho(L) \le \norm{L} < 1$, i.e., $K$ is stable.
Notice that for a strongly stable $K$, although $\rho(A+BK)<1$, it
may not be the case that $\norm{A+BK} < 1$, and a non-trivial
transformation $H \ne I$ may be required to make the norm smaller
than one (this is indeed the case with feasible solutions to our SDP
relaxation).

Strong stability ensures exponentially fast convergence to steady-state, as is made precise in the next lemma.

\begin{lemma} \label{lem:stab-mixing}
For all $t=1,2,\ldots$ let $\wh X_t$ be the state covariance matrix on round $t$ starting from some $\wh X_0 \succeq 0$ and following a ($\kappa,\gamma$)-strongly stable policy $\pi(x) = Kx$.
Then $\wh X_1, \wh X_2,\ldots$ approaches a steady-state covariance matrix $X$,
and further, for all $t$ it holds that
\begin{align*}
\norm{\wh X_t - X}
&\le
\kappa^2 e^{-2\gamma t} \norm{\wh X_0 - X}
.
\end{align*}
This exponential convergence is true even if the policy is randomized and follows $K$ in expectation; that is, if $\E[\pi(x) | x] = Kx$, and provided that $\cov[\pi(x)|x]$ is finite.
\end{lemma}

\begin{proof}
Let us first analyze deterministic policies.
As noted above, we know that $K$ is stable and as a result the state covariances $\wh X_t$ approach a steady-state covariance $X$.
By definition, we have
\begin{alignat*}{2}
&\wh X_{t+1}&
&=
(A+BK) \wh X_t (A+BK)\tr + W
\qquad
\forall ~ t \ge 0
;
\\
&X&
&=
(A+BK) X (A+BK)\tr + W
.
\end{alignat*}
Subtracting the equations and recursing, we have
$
\wh X_t - X
=
(A+BK)^t (\wh X_0 - X) ((A+BK)^t)\tr
,
$
which gives
\begin{align*}
\norm{\wh X_t - X}
\le
\norm{(A+BK)^t}^2 \norm{\wh X_0 - X}
.
\end{align*}
For further bounding the right-hand side, observe that $(A+BK)^t = H L^t H^{-1}$, thus
\begin{align*}
\norm{(A+BK)^t}
\le
\norm{H} \norm{H^{-1}} \norm{L}^t
\le
\kappa (1-\gamma)^{t}
\le
\kappa e^{-\gamma t}
.
\end{align*}
Combining the inequalities gives the result for deterministic policies.

For randomized policies with $\E[u | x] = K x$ and finite $V = \cov[u | x]$, the dynamics of the state covariance take the form
\begin{alignat*}{2}
&\wh X_{t+1}&
&=
(A+BK) \wh X_t (A+BK)\tr + BVB\tr + W
\quad
\forall ~ t \ge 0
;
\\
&X&
&=
(A+BK) X (A+BK)\tr + BVB\tr + W
.
\end{alignat*}
Since the analysis above only depends on the difference between the equations, the added $BVB\tr$ term has no effect on the convergence of $X_t$.
Note, however, that the steady state $X$ itself will be a function of $V$ in general.
\end{proof}

Let us state one more property of strongly stable policies that will be useful in our analysis.

\begin{lemma} \label{lem:stab-trbound}
Assume that $K$ is $(\kappa,\gamma)$-strongly stable, and let $X$ and $U$ be the covariances of $x$ and $u$ at steady-state when following $K$.
Then $\trace(X) \le (\kappa^2/\gamma) \trace(W)$ and $\trace(U) \le (\kappa^4/\gamma) \trace(W)$.
\end{lemma}

\subsection{Sequential strong stability}

We next present a stronger notion of strong stability which plays a central role in our analysis.
Roughly speaking, the goal is to argue about fast mixing when following a sequence of different policies $K_1,K_2,\ldots$ (rather than a fixed policy $K$ throughout).
In this case, for any kind of mixing to take place, not only does one has to require that each policy is strongly stable, but also that the sequence is ``slowly changing.''
This motivates the following definition.

\begin{definition}[sequential strong stability]
A sequence of policies $K_1,\ldots,K_T$ is \emph{$(\kappa,\gamma)$-strongly stable} (for $\kappa > 0$ and $0 < \gamma \le 1$) if there exist matrices $H_1,\ldots,H_T$ and $L_1,\ldots,L_T$ such that $A+BK_t = H_t L_t H_t^{-1}$ for all $t$, with the following properties:
\begin{enumerate}[nosep,label=(\roman*)]
\item
$\norm{L_t} \le 1-\gamma$ and $\norm{K_t} \le \kappa$;
\item
$\norm{H_t} \le \beta$ and $\norm{H_t^{-1}} \le 1/\alpha$ with $\kappa = \beta/\alpha$;
\item
$\norm{H_{t+1}^{-1} H_{t}^{}} \le 1+\gamma/2$.
\end{enumerate}
\end{definition}

Strongly stable sequences mix quickly, in the following~sense (proof is deferred to \cref{sec:proofs}).

\begin{lemma} \label{lem:seqstab-mixing}
Let $\pi_t(x) = K_t x$ ($t=1,2,\ldots$) be a sequence of policies with respective steady-state covariance matrices $X_1,X_2,\ldots$, such that $K_1,K_2,\ldots$ is a $(\kappa,\gamma)$-strongly stable sequence and $\norm{X_t - X_{t-1}} \le \eta$ for all $t$, for some $\eta>0$.
Let $\wh{X}_t$ be the state covariance matrix on round $t$ starting from some $\wh{X}_1 \succeq 0$ and following this sequence.
Then
\begin{align*}
\norm{\wh X_{t+1} - X_{t+1}}
\le
\kappa^2 e^{-\gamma t} \norm{\wh{X}_1 - X_1} + \frac{2\eta\kappa^2}{\gamma}
.
\end{align*}
The same is true even if the policies are randomized, such that $\E[\pi_t(x) | x] = K_t x$ and $\cov[\pi_t(x) | x]$ exists and is finite.
\end{lemma}

\section{SDP Relaxation for LQ control}
\label{sec:sdp}

We now present our SDP relaxation for the infinite-horizon LQ control problem.
Our presentation requires the following definitions.
Consider an LQ control problem parameterized by matrices $A,B,Q,R$ and $W$.
For any stable policy (for which a steady-state distribution exists), define
\begin{align} \label{eq:SigPi}
\SigK(\pi)
=
\E\begin{pmatrix} xx\tr & xu\tr \\ ux\tr & uu\tr \end{pmatrix}
,
\end{align}
where $x$ is distributed according to the steady-state distribution of $\pi$, and $u = \pi(x)$.
Then, the infinite horizon cost of $\pi$ is given by $\cost(\pi) = (\begin{smallmatrix} Q & 0 \\ 0 & R \end{smallmatrix}) \bullet \SigK(\pi)$.
%
For a policy $\pi_K(x) = Kx$ defined by a stable control matrix $K$ (i.e., for which $\rho(A+BK) < 1$),
this matrix takes the form
\begin{align} \label{eq:SigK}
\SigK(K)
=
\begin{pmatrix} X & XK\tr \\ KX & KXK\tr \end{pmatrix}
,
\end{align}
where $X$ is the state covariance at steady-state.
(We slightly abuse notation and write $\SigK(K)$ instead of $\SigK(\pi_K)$).
In this case, one also has $\cost(K) = \cost(\SigK(K)) = (Q + K\tr R K) \bullet X.$

\subsection{The relaxation}
\label{sec:sdprelaxation}

We can now present our SDP relaxation for the LQ control problem given by $(A,B,Q,R,W)$, which takes the form:
\begin{align}
&\mbox{minimize} && \cost(\Sigma) = \begin{pmatrix} Q & 0 \\ 0 & R \end{pmatrix} \bullet \Sigma \nonumber \\
&\mbox{subject to}&&\Sigma_{xx} = \begin{pmatrix}
A & B
\end{pmatrix} \Sigma \begin{pmatrix}
A & B
\end{pmatrix}\tr + W , \label{eq:sdpconstraint} \\
&&&\Sigma \succeq 0, ~~ \trace(\Sigma) \le \trbound \nonumber
.
\end{align}
Here,
$\trbound>0$ is a parameter whose value will be determined later, and $\Sigma$ is a $(d+k) \times (d+k)$ symmetric matrix that decomposes to blocks as follows:
$$
\Sigma = \begin{pmatrix}
\Sigma_{xx} & \Sigma_{xu} \\
\Sigma_{xu}\tr & \Sigma_{uu}
\end{pmatrix}
,
$$
where $\Sigma_{xx}$ is a $d \times d$ block, $\Sigma_{uu}$ is $k \times k$, and $\Sigma_{xu}$ is $d \times k$.

The program \cref{eq:sdpconstraint} is a relaxation in the following sense.

\begin{lemma}
For any stable policy $\pi$ such that at steady-state $\E\norm{x}^2 + \E\norm{u}^2 \le \trbound$,
the matrix $\Sigma = \SigK(\pi)$ is feasible for \eqref{eq:sdpconstraint}.
\end{lemma}

\begin{proof}
Let $\pi$ be any stable policy and consider the matrix $\Sigma = \SigK(\pi)$.
Then $\Sigma \succeq 0$ (by definition, recall \cref{eq:SigPi}), and satisfies the equality constraint of \eqref{eq:sdpconstraint}, since if $x$ is at steady-state and $u=\pi(x)$, then $Ax+Bu + w$ has the same distribution as $x$ for $w \sim \calN(0,W)$ independent of $x$ and $u$, thus
$
\E[xx\tr]
=
\E[(Ax+Bu + w)(Ax+Bu + w)\tr] ;
$
the latter is equivalent to $\Sigma_{xx} = (\begin{matrix} A & B \end{matrix}) \Sigma (\begin{matrix} A & B \end{matrix})\tr + W$.
Finally, observe that
$
\trace(\Sigma)
=
\E\trace(xx\tr)+\E\trace(uu\tr)
=
\E\norm{x}^2 + \E\norm{u}^2
$
where $x,u$ are distributed according to the steady-state distribution of $\pi$, hence $\Sigma$ satisfies the trace constraint.
\end{proof}

%
%
%


\subsection{Extracting a policy}

We next show that from any feasible solution to the SDP, one can extract a stable policy with the same (if not better) cost, provided that $W \succ 0$.
For any feasible solution $\Sigma$ for the SDP, define a control matrix as follows:
\begin{align} \label{eq:KSig}
\KSig(\Sigma) = \Sigma_{xu}\tr \Sigma_{xx}^{-1}
.
\end{align}
Note that, due to the equality constraint of the SDP, our assumption $W \succ 0$ ensures that $\Sigma_{xx} \succ 0$, thus $\Sigma_{xx}$ is nonsingular and $\KSig(\Sigma)$ is well defined.

\begin{theorem} \label{thm:SigKSig}
Let $\Sigma$ be any feasible solution to the SDP, and let $K = \KSig(\Sigma)$.
Then the policy $\pi(x) = Kx$ is stable, and it holds that $\SigK(K) \preceq \Sigma$.
In particular, $\SigK(K)$ is also feasible for the SDP and its cost is at most that of $\Sigma$.
\end{theorem}

Without the trace constraint, the theorem particularly implies that for the optimal solution $\Sigma^\st$ of the SDP, the corresponding control matrix $K^\st = \KSig(\Sigma^\st)$ is an optimal policy for the original problem, recovering a classic result in control theory.

\begin{proof}[Proof of \cref{thm:SigKSig}]
Our first step is to show that
\begin{align} \label{eq:SigSigpr}
\Sigma
\succeq
\Sigma'
=
\begin{pmatrix} \Sigma_{xx} & \Sigma_{xx} K\tr \\ K \Sigma_{xx} & K \Sigma_{xx} K\tr \end{pmatrix}.
\end{align}
To see this, observe that by definition of $K = \KSig(\Sigma)$ we have
\begin{align*}
\Sigma = \Sigma' +
\begin{pmatrix}
0 & 0 \\
0 & \Sigma_{uu} - \Sigma_{ux}\tr \Sigma_{xx}^{-1} \Sigma^{}_{ux}
\end{pmatrix}
~.
\end{align*}
Thus, it suffices to show that $\Sigma_{uu} - \Sigma_{ux}\tr \Sigma_{xx}^{-1} \Sigma^{}_{ux}$ is PSD.
The latter matrix is the Schur complement of $\Sigma$, and is PSD because $\Sigma$ is PSD.

Next, we show that the control matrix $K$ gives rise to a stable policy.
Let us develop \cref{eq:sdpconstraint}.
First, since $W \succ 0$ we also have that $\Sigma_{xx} \succ 0$.
Moreover, by \cref{eq:SigSigpr},
\begin{align*}
\Sigma_{xx} &= (\begin{matrix} A & B \end{matrix}) \Sigma (\begin{matrix} A & B \end{matrix})\tr + W \\
&\succeq (A+BK) \Sigma_{xx} (A + BK)\tr + W \\
&\succ (A+BK) \Sigma_{xx} (A + BK)\tr~.
\end{align*}
Let $\lambda$ and $v$ be a (possibly complex) eigenvalue and left-eigenvector associated with $A + B K$. Then,
\[
v^* \Sigma_{xx} v
>
v^* (A+BK) \Sigma_{xx} (A + BK)\tr v
=
\abs{\lambda}^2 v^* \Sigma_{xx} v~,
\]
which, by $v^* \Sigma_{xx} v > 0$, implies $\abs{\lambda} < 1$.
This is true for all eigenvalues $\lambda$, and shows that $\rho(A+BK) < 1$, that is, $K$ is~stable.

Finally, let us show that $\SigK(K) \preceq \Sigma'$, which together with \cref{eq:SigSigpr} would imply our claim $\SigK(K) \preceq \Sigma$.
Denote by $X$ the state covariance at steady-state when following $K$; then,
\begin{align*}
\SigK(K)
=
\begin{pmatrix} X & X K\tr \\ K X & K X K\tr \end{pmatrix}
.
\end{align*}
To establish that $\SigK(K) \preceq \Sigma'$ it is enough to show $X \preceq \Sigma_{xx}$.
To this end, let $\Delta = \Sigma_{xx}-X$ and write
\begin{align*}
X + \Delta &\succeq (A+BK) X (A + BK)\tr + W \\
&\qquad + (A+BK) \Delta (A+BK)\tr \\
&=
X + (A+BK) \Delta (A+BK)\tr
~,
\end{align*}
from which we get $\Delta \succeq  (A+BK) \Delta (A+BK)\tr$.
Applying the latter inequality recursively, we~obtain
\begin{align*}
\Delta
\succeq  (A+BK)^n \Delta ((A+BK)\tr)^n~.
\end{align*}
Recall that $\rho(A+BK) < 1$; thus, taking the limit as $n \to \infty$, we get
$
(A+BK)^n \Delta ((A+BK)\tr)^n \to 0 ,
$
which implies $\Delta \succeq 0$.
This shows that $X \preceq \Sigma_{xx}$, as required.

To complete the proof observe that $\SigK(K)$ is feasible for the SDP since $\SigK(K) \preceq \Sigma$ and $\Sigma$ is feasible.
Furthermore, since $(\begin{smallmatrix} Q & 0 \\ 0 & R \end{smallmatrix})$ is PSD, we have
\begin{align*}
\cost(\SigK(K))
=
\begin{pmatrix} Q & 0 \\ 0 & R \end{pmatrix} \bullet \SigK(K)
\le
\begin{pmatrix} Q & 0 \\ 0 & R \end{pmatrix} \bullet \Sigma
=
\cost(\Sigma)
.
&\qedhere
\end{align*}
\end{proof}

\subsection{Strong stability of solutions}
\label{sec:sdpstrongstability}

Let us show that from a solution to the SDP one can extract a strongly stable policy.

\begin{lemma} \label{lem:sdp-sstable}
Assume that $W \succeq \sigma^2 I$ and let $\kappa=\sqrt{\trbound}/\sigma$.
Then for any feasible solution $\Sigma$ for the SDP, the policy $K = \KSig(\Sigma)$ is $(\kappa, 1/2\kappa^2)$-strongly stable.
\end{lemma}

%
\begin{proof}
According to \cref{thm:SigKSig}, the policy $K$ is (weakly) stable and the matrix $\wh\Sigma = \SigK(K)$ is feasible for the SDP.
Let $X = \wh\Sigma_{xx}$ be the state covariance of $K$ at steady-state.
Since $\wh\Sigma$ is feasible, and since $W \succeq \sigma^2 I$, we have
\begin{align} \label{eq:XW0}
X
\succeq
(A+BK) X (A+BK)\tr + \sigma^2 I
.
\end{align}
In particular, this means that $X \succeq \sigma^2 I$.
On the other hand, we have $\trace(X) \le \trace(\wh\Sigma) \le \trbound$,
thus $X \preceq \trbound I$.
Overall, 
\begin{align} \label{eq:Xlbub}
\sigma^2 I \preceq X \preceq \trbound I
.
\end{align}
Given that $X$ is nonsingular, we can define $L = X^{-1/2}(A+BK) X^{1/2}$.
Multiplying \cref{eq:XW0} by $X^{-1/2}$ from both sides, we obtain
$
I
\succeq
L L\tr + \sigma^2 X^{-1}
\succeq
L L\tr + \kappa^{-2} I
.
$
Thus $LL\tr \preceq (1-\kappa^{-2}) I$, so $\norm{L} \le \sqrt{1-\kappa^{-2}} \le 1-\kappa^{-2}/2$.
Also, \cref{eq:Xlbub} shows that $\norm{X^{1/2}}\norm{X^{-1/2}} \le \kappa$.
%
It is left to establish the bound on the norm $\norm{K}_\frob$.
To this end, use the fact that
\begin{align*}
X \bullet KK\tr
=
\trace(KXK\tr)
=
\trace(\wh\Sigma_{uu})
\le
\trbound
\end{align*}
together with $X \succeq \sigma^2 I$ (recall \cref{eq:Xlbub}) to obtain $\sigma^2 \norm{K}_\frob^2 \le \trbound$, that is, $\norm{K}_\frob \le \kappa$.
\end{proof}

We can also prove an analogous statement for sequences of feasible solutions, provided that they change slowly enough
(we defer the proof to \cref{sec:proofs}).

\begin{lemma} \label{lem:sdp-seqstab}
Assume that $W \succeq \sigma^2 I$ and let $\kappa=\sqrt{\trbound}/\sigma$.
Let $\Sigma_1,\Sigma_2,\ldots$ be a sequence of feasible solutions of \eqref{eq:sdpconstraint}, and suppose that $\norm{\Sigma_{t+1}-\Sigma_{t}} \le \eta$ for all $t$ for some $\eta \le \sigma^2/\kappa^2$.
Then the sequence $K_1,K_2,\ldots$, where $K_t = \KSig(\Sigma_t)$ for all $t$ is $(\kappa,1/2\kappa^2)$-strongly stable.
\end{lemma}

\section{Online LQ Control}

\begin{algorithm}
\begin{algorithmic}
\STATE \textbf{Parameter}: $\eta,\nu > 0$
\STATE Initialize $\Sigma_1 = I_{n \times n}$ with $n = d+k$
\FOR{$t = 1,2,\ldots$}
\STATE Receive state $x_t$
\STATE Compute $K_t = (\Sigma_t)_{ux} (\Sigma_t)_{xx}^{-1}$, $V_t = (\Sigma_t)_{uu} - K_t (\Sigma_t)_{xx} K_t\tr$
\STATE Predict $u_t \sim \calN(K_t x_t, V_t)$; receive $Q_t$, $R_t$
\STATE Update:
$$\Sigma_{t+1} = \Pi_\cS\Lrbra{\Sigma_{t} - \eta \Lr{\begin{smallmatrix}Q_t & 0 \\ 0 & R_t\end{smallmatrix}}},$$
where $\Pi_\cS$ is the Frobenius-norm projection onto
\begin{align*}
\cS = \Bigg\{~\Sigma \in \reals^{n \times n} ~\,\Bigg|~~
\begin{aligned}
&\Sigma \succeq 0, ~~ \trace(\Sigma) \le \trbound, \\[-1ex]
&\Sigma_{xx} = \begin{pmatrix} A & B \end{pmatrix} \Sigma \begin{pmatrix} A & B \end{pmatrix}\tr + W
\end{aligned}~\Bigg\}
\end{align*}
\ENDFOR
\end{algorithmic}
\caption{Online LQ Controller}
\label{alg:ogd}
\end{algorithm}

In this section we describe our gradient based algorithm for online LQ control, presented in \cref{alg:ogd}.
The algorithm maintains an ``ideal'' steady-state covariance matrix $\Sigma_t$ by performing online gradients steps directly on the SDP we formulated in \cref{sec:sdp} (with the linear cost functions changing from round to round).
Then, a control matrix $K_t$ is extracted from the covariance $\Sigma_t$ and is used to generate a prediction.

Notice that the predictions made by the algorithm are randomly drawn from the Gaussian $\calN(K_t x_t, V_t)$, and only follow the extracted policies $K_1,K_2,...$ in expectation.
This randomization step is crucial for the algorithm to exhibit fast mixing: sampling the prediction from a distribution with the right covariance ensures the observed covariance matrices converge to those generated by the algorithm, and consequently this sequence ``mixes'' more quickly.

For \cref{alg:ogd} we prove the following guarantee.

\begin{theorem} \label{thm:ogd}
Assume that $\trace(W) \le \lambda^2$ and $W \succeq \sigma^2 I$.
Given $\kappa>0$ and $0 \le \gamma < 1$, set $\trbound = 2\kappa^4\lambda^2/\gamma$ and $\eta = \sigma^3/(2C\sqrt{\trbound T})$.
The expected regret of \cref{alg:ogd} compared to any $(\kappa,\gamma)$-strongly stable control matrix $K^\st$ is at most
\begin{align*}
\cost_T(A) - \cost_T(K^\st)
=
\O\lr{ \frac{\kappa^{10} \lambda^{5}}{\gamma^{2.5}\sigma^{3}} C \sqrt{T} }
,
\end{align*}
provided that $T \ge 8 \kappa^4\lambda^2/(\gamma\sigma^2)$.
\end{theorem}

We remark that the theorem (in fact, \cref{alg:ogd} itself) tacitly assumes that the SDP defined by $\cS$ is feasible; otherwise, the set of strongly-stable policies is empty and the statement of \cref{thm:ogd} is vacuous.

\begin{proof}
Fix an arbitrary $(\kappa,\gamma)$-strongly stable control matrix $K^\st$, and denote by $\wh\Sigma^\st_1,\ldots,\wh\Sigma^\st_T$ be the covariances induced by using $K^\st$ throughout.
Also, let $\wh\Sigma_1,\ldots,\wh\Sigma_T$ be the actual observed covariance matrices induced by the algorithm.
Denoting $L_t = \Lr{\begin{smallmatrix} Q_t & 0 \\ 0 & R_t \end{smallmatrix}},$ the expected regret of the algorithm can be then written as follows:
\begin{align} \label{eq:regretdecomp1}
\sum_{t=1}^T L_t \bullet (\wh\Sigma_t - \wh\Sigma^\st_t)
&=
\sum_{t=1}^T L_t \bullet (\wh\Sigma_t - \Sigma_t)
\notag\\
&+
\sum_{t=1}^T L_t \bullet (\Sigma_t - \Sigma^\st)
\\
&+
\sum_{t=1}^T L_t \bullet (\Sigma^\st - \wh\Sigma^\st_t)
.
\notag
\end{align}
Observe that the sequence $\Sigma_1,\ldots,\Sigma_T$ generated by the algorithm is feasible for the (feasibility) SDP described by the set $\cS$.
Thanks to \cref{lem:sdp-sstable}, for any feasible $\Sigma \in \cS$ the corresponding control matrix $\KSig(\Sigma)$ is $(\bar\kappa,\bar\gamma)$-strongly stable, for $\bar\kappa=\sqrt{\trbound}/\sigma$ and $\bar\gamma=\sigma^2/2\trbound$; in particular, this applies to each of the matrices $\Sigma_t$.

We proceed by bounding each of the sums on the right-hand side of \cref{eq:regretdecomp1}. We start with the second term and use a well-known regret bound for the Online Gradient Descent algorithm, due to \citet{zinkevich2003online}.

\begin{lemma} \label{lem:ogd-regret}
We have
\[
\sum_{t=1}^T L_t \bullet (\Sigma_t - \Sigma^\st)
\le
\frac{4\trbound^2}{\eta} + 4C^2 \eta T
.
\]
Additionally, the $\Sigma_t$ are slowly changing in the sense that, for all $t$,
\begin{equation}
\norm{\Sigma_{t+1} - \Sigma_t}_\frob
\le
4C\eta
.\label{eq:slowchangingSigma}
\end{equation}
\end{lemma}

\short{
\begin{proof}
The diameter of the feasible domain of the SDP (with respect to $\norm{\cdot}_\frob$) is upper bounded by $2\trbound$.
Also, the linear loss function $X
\mapsto L_t \bullet X$ is $\norm{Q_t}_\frob + \norm{R_t}_\frob \le 2C$ Lipschitz for all $t$ (again, with respect to $\norm{\cdot}_\frob$).
Plugging this into the regret bound of the Online Gradient Descent algorithm gives the lemma.
\end{proof}
}

We next bound the first term, now relying on \cref{eq:slowchangingSigma} and the fact
that the sequence of (randomized) policies chosen by \cref{alg:ogd} is strongly stable.
\begin{lemma} \label{lem:ogd-mixing}
If $\eta \le \sigma^2/4C\bar\kappa^2$, it holds that
\begin{align*}
\sum_{t=1}^T L_t &\bullet (\wh\Sigma_t - \Sigma_t)
\le
\frac{16C^2\bar\kappa^4}{\bar\gamma} \eta T + \frac{4C\bar\kappa^4}{\bar\gamma} \trbound
.
\end{align*}
\end{lemma}

Finally, the last term in \cref{eq:regretdecomp1} can be bounded using the strong stability of $K^\st$.

\begin{lemma} \label{lem:Kst-mixing}
For any $(\kappa,\gamma)$-strongly stable $K^\st$,
\begin{align*}
\sum_{t=1}^T L_t \bullet (\Sigma^\st - \wh\Sigma^\st_t)
\le
2C \frac{\kappa^4 \trbound}{\gamma}
.
\end{align*}
\end{lemma}

\short{
\begin{proof}
Denote $X = \Sigma^\st_{xx}$ and $X_t = (\Sigma^\st_t)_{xx}$, and observe that
\begin{align*}
\Sigma^\st
=
\begin{pmatrix}
X & X (K^\star)\tr \\
K^\star X & K^\star X (K^\star)\tr
\end{pmatrix}
,~~
\wh\Sigma^\st_t
=
\begin{pmatrix}
X_t & X_t (K^\star)\tr \\
K^\star X_t & K^\star X_t (K^\star)\tr
\end{pmatrix}
.
\end{align*}
Thus $L_t \bullet (\Sigma^\st - \wh\Sigma^\st_t) = (Q_t + (K^\st)\tr R_t K^\st) \bullet (X - X_t)$.
Now, \cref{lem:stab-trbound} asserts that $\trace(\Sigma^\st) \le 2\kappa^4/\gamma = \trbound$, hence $\Sigma^\st \in \cS$ and by \cref{eq:QKRKbound}, it follows that
\begin{align*}
\sum_{t=1}^T L_t \bullet (\Sigma^\st - \wh\Sigma^\st_t)
\le
2C\kappa^2 \sum_{t=1}^T \norm{X - X_t}
~.
\end{align*}
Now, an application of \cref{lem:stab-mixing} gives
\begin{align*}
\sum_{t=1}^T \norm{X_t - X}
\le
\kappa^2 \norm{X_1 - X} \sum_{t=1}^T e^{-2\gamma t}
\le
\frac{\kappa^2}{2\gamma} \norm{X_1 - X}
~,
\end{align*}
where in the ultimate inequality we have used again the fact that $\sum_{t=1}^T e^{-\gamma t} \le 1/\gamma$.
Finally, we have
$
\norm{X_1 - X}
\le
\norm{\Sigma^\st - \wh\Sigma^\st_0}
\le
2\trbound
.
$
Combining the inequalities gives the result.
\end{proof}
}

The theorem now follows by plugging in the bounds we established in \cref{lem:ogd-mixing,lem:ogd-regret,lem:Kst-mixing} into \cref{eq:regretdecomp1} and setting our choices of $\eta$ and $\trbound$.
(See \cref{sec:proofs} for details.)
\end{proof}

\section{Oracle-based Algorithm}

In this section we present a different approach that is based on
Follow the Lazy Leader of \cite{kalai2005efficient}. In contrast to
\cref{alg:ogd}, this approach does not require a lower bound on the
noise but rather relies on occasionally performing resets, and needs
a bound on the cost of this reset (this is established in
\cref{sec:reset_cost} under reasonable assumptions).
We assume access to an \textsc{Oracle} procedure that receives cost matrices $Q$, $R$, and parameter $\nu > 0$. It returns a control matrix $K$ that minimizes the steady-state cost, subject to $\trace(X) + \trace(KXK\tr) \le \nu$, where $X$ is the steady-state covariance matrix associated with $K$.\footnote{\textsc{Oracle} can be implemented by solving the SDP in \cref{sec:sdp}.}

\begin{algorithm}[ht]
\caption{Follow the Lazy Leader}
\begin{algorithmic}
\label{alg:fpl}
\STATE \textbf{Parameter}: $\eta,\nu > 0$, transition matrices $A$, $B$, distribution $\mu$.
\STATE Sample $Q^p_1 \in \R^{d \times d}$,$R^p_1 \in \R^{k \times k}$ from $d \mu$.
\STATE Set $\widehat{Q}_1 \gets 0$, $\widehat{R}_1 \gets 0$
\FOR{$t = 1,2,\ldots$}
\STATE Receive state $x_t$.
\STATE Compute $K_t \gets \textsc{Oracle}(\wh Q_t + Q^p_t,\wh R_t + Q^p_t, \nu)$.
\STATE Predict $u_t \gets K_t x_t$.
\STATE Receive $Q_t$,$R_t$.
\STATE Update $\widehat{Q}_{t+1} = \widehat{Q}_t + Q_t$, $\widehat{R}_{t+1} = \widehat{R}_t + R_t$.
\STATE With probability $\min\left\{1, \frac{d\mu(Q^p_t - Q_t, R^p_t - R_t)}{d\mu (Q^p_t, R^p_t)} \right\}$, set
\STATE \qquad $Q^p_{t+1} \gets Q^p_t - Q_t$.
\STATE \qquad $R^p_{t+1} \gets R^p_t - R_t$,
\STATE else, perform reset and set
\STATE \qquad $Q^p_{t+1} \gets -Q^p_t$.
\STATE \qquad $R^p_{t+1} \gets -R^p_t$.
\ENDFOR
\end{algorithmic}
\end{algorithm}

\cref{alg:fpl} is similar to Follow the Perturbed Leader, and in fact behaves the same in expectation. At every round $t$, \textsc{Oracle} is called using the sum of previously seen $Q$s and $R$s plus an additional random noise, $Q^p_t$ and $R^p_t$. \textsc{Oracle} returns a matrix $K_t$ that is used to choose $u_t = K_t x_t$.

For the measure $d \mu$, we use the joint measure over symmetric matrices $Q$ and $R$, whose upper triangle is sampled coordinate-wise i.i.d from Laplace($1/\eta$).
The "lazyness" of the algorithm stems from $Q^p_1,\ldots,Q^p_T$ and $R^p_1,\ldots,R^p_T$ being sampled dependently over time such that the cumulative perturbed loss only changes with small probability between rounds. Consequently, the expected number of switches of $K$ as well as the expected number of resets are only $O(\eta T)$.

The reset step in the algorithm, informally, drives the system to zero at some cost. Here we assume that $B$ has full column-rank in which case we can reset in one step. In \cref{sec:reset_cost}, we show how resetting can be done over a sequence of steps under much weaker assumptions.

\begin{observation}
Suppose that $B$ has full column-rank. Resetting the system in round $t$ can be done by setting $u_t = -B^\dagger A x_t$, such that
at the next round $x_{t+1} = w_{t+1}$.
Moreover, the expected cost of the reset is at most~$C \nu (1+\|B^\dagger A\|^2)$.
\end{observation}

For \cref{alg:fpl} we will show the following regret bound.

\begin{theorem}
\label{thm:fplanalysis}
Assume that $\trace(W) \le \lambda^2$, and suppose that the cost of a reset is at most $C_r$. Then for $\trbound = 2\kappa^4\lambda^2/\gamma$, the expected regret of \cref{alg:fpl} against any $(\kappa,\gamma)$-strongly-stable control matrix $K^\st$ satisfies
\begin{align*}
\E \lrbra{ \cost_T(A) - \cost_T(K^\st) }
=
O \Lr{ (d+k)^{3/4} \sqrt{C \nu(C_r + C \nu) T} }
.
\end{align*}
\end{theorem}

\begin{remark}
\textsc{Oracle} requires that the matrices $Q$ and $R$ are PSD. Nonetheless, we invoke \textsc{Oracle} using the perturbed cumulative loss $(\hat Q_t + Q^p_t, \hat R_t + R^p_t)$ that might not be PSD, as the perturbations $Q^p_t$ and $R^p_t$ themselves are typically not PSD.
To solve this issue, we first notice that with high-probability
\citep{vershynin2010introduction}, we have $\|Q^p_t\| \le O(d/\eta)$
and $\|R^p_t\| \le O(k/\eta)$. Therefore, to guarantee that the
perturbed cumulative loss is PSD, we can add an initial large
pretend loss by setting $\wh Q_1 = (d/\eta)I$ and $\wh R_1 =
(k/\eta)I$. This would contribute an $O(C\nu (d+k)/\eta)$ term to
the regret which ensures that, by our choice of $\eta$,
\cref{thm:fplanalysis} still~holds.
\end{remark}


\begin{proof}[Proof of \cref{thm:fplanalysis}]
Let $\wh X_1,\ldots,\wh X_T$ be the actual observed covariance matrices induced by \cref{alg:fpl}.
Also, let $\wh X^\st_1,\ldots,\wh X^\st_T$ be the covariances induced by using a fixed control matrix $K^\st$ throughout. Similarly, define $X_1,\ldots,X_T$ to be the covariance matrices of the steady-state distributions induced by $K_1,\ldots,K_T$ respectively, and $X^\st$ that of $K^\st$.

As in the analysis of OGD, the expected regret can be decomposed as follows:
\begin{align} \label{eq:regretdecomp2}
&\sum_{t=1}^T (Q_t + K_t\tr R_t K_t) \bullet \wh X_t - (Q_t + (K^\st)\tr R_t K^\st) \bullet \wh X_t^\st \notag \\
&\qquad =
\sum_{t=1}^T (Q_t + K_t\tr R_t K_t) \bullet (\wh X_t - X_t)
\notag\\
&\qquad +
\sum_{t=1}^T (Q_t + K_t\tr R_t K_t) \bullet X_t - (Q_t + (K^\st)\tr R_t K^\st) \bullet X^\st
\notag\\
&\qquad+
\sum_{t=1}^T (Q_t + (K^\st)\tr R_t K^\st) \bullet (X^\st - \wh X^\st_t)
.
\end{align}

The second term in \cref{eq:regretdecomp2}, the regret in the ``idealized setting'', is bounded due to \citet{kalai2005efficient}. It requires the additional observation that, by \cref{lem:stab-trbound}, we have $\trace(X^\st) + \trace(K^\st X^\st (K^\st)\tr) \le \nu$.

\begin{lemma}
\label{prop:kv}
Assume $\trace(Q_t), \trace(R_t) \le C$ for all $t$.
Then,
\begin{align*}
&\E \left[\sum_{t=1}^T \trace(X_t (Q_t + K_t\tr R_t K_t)) - \trace(X^\star (Q_t + (K^\star)\tr R_t K^\star)) \right]\\
&\qquad \le 8 \eta C^2 \nu \sqrt{d+k} T + \frac{16 \nu (d+k)}{\eta}~.
\end{align*}
Moreover, the probability that the algorithm changes $K_t$ and performs a reset  at any step $t$ is at most $\eta C \sqrt{d+k}$.
\end{lemma}

The third term of \cref{eq:regretdecomp2} is bounded by $2C \kappa^4 \nu / \gamma$ due to \cref{lem:Kst-mixing}. It remains to bound the first term in the equation. To that end, we will next show that after the system is reset, the cost of the learner on round $t$ is at most that of the steady-state induced by $K_t$.

\begin{lemma}
\label{lemma:ssafterreset}
Suppose the learner starts playing $K$ at state $x_{t_0} = w_{t_0}$. Then the expected cost of the learner is always less then the steady-state cost induced by $K$.
\end{lemma}

\begin{proof}
Let $x_{t_0} = w_{t_0}$, and recall that $x_{t+1} = (A + BK)x_t + w_t$. Let $\wh X_t$ be the covariance of $x_t$, and $X$ be the covariance of $x$ at the steady-state induced by $K$. Then, $X_{t_0} = (A + BK_{t_0})X_{t_0}(A+BK_{t_0})\tr + W$.

We now show that $\wh X_t \preceq X$ for all $t \ge t_0$ by induction. Indeed, for the base case $\wh X_{t_0} = W \preceq (A+BK_{t_0})X_{t_0}(A+BK_{t_0})\tr + W = X$. Now assume that $\wh X_t \preceq X_{t_0}$, that implies
\begin{align*}
\wh X_{t+1} &= (A + BK_{t_0})\wh X_t (A+BK_{t_0})\tr + W\\
&\preceq (A + BK_{t_0})X_{t_0} (A+BK_{t_0})\tr + W = X~.
\end{align*}

Since $Q_t + K_t\tr R_t K_t$ is PSD, the expected cost of the learner at time $t$ is $(Q_t + K_t\tr R_t K_t) \bullet X_t \le (Q_t + K_t\tr R_t K_t) \bullet X$.
\end{proof}

Combining \cref{lemma:ssafterreset,prop:kv} obtains the theorem (see \cref{sec:proofs} for more details).
\end{proof}

\short{
\begin{proof}[Proof of \cref{thm:fplanalysis}]
Since after a reset, the system starts at state $0$, the cost of the learner is always less than the steady-state cost. The expected number of switches is at most $\eta C \sqrt{d+k} T$, and whenever a switch occurs we pay an additional cost of $C_r$ for performing a reset. Combining that with our bounds on the three terms in \cref{eq:regretdecomp2}, we get
\begin{align*}
\cost_T(A) - \cost_T(K^\st) &\le \eta C C_r \sqrt{d+k} T + 8 \eta C^2 \nu \sqrt{d+k} T \\
&\quad+ \frac{16 \nu \log(d+k)}{\eta} + \frac{2C \kappa^4 \nu}{\gamma}~.
\end{align*}
Setting $\eta = \sqrt{16 \nu \log(d+k) / TC\sqrt{d+k} (C_r + 8C \nu)}$, and plugging in $\nu = 2 \kappa^4 \lambda^2/\gamma$ completes the theorem.
\end{proof}
}

%
%

\section{Experiments}
\label{sec:experiments}

\begin{figure}
\centering
\includegraphics[trim={0.5cm 3.1cm 2cm 1.5cm},clip,width=0.65\linewidth]{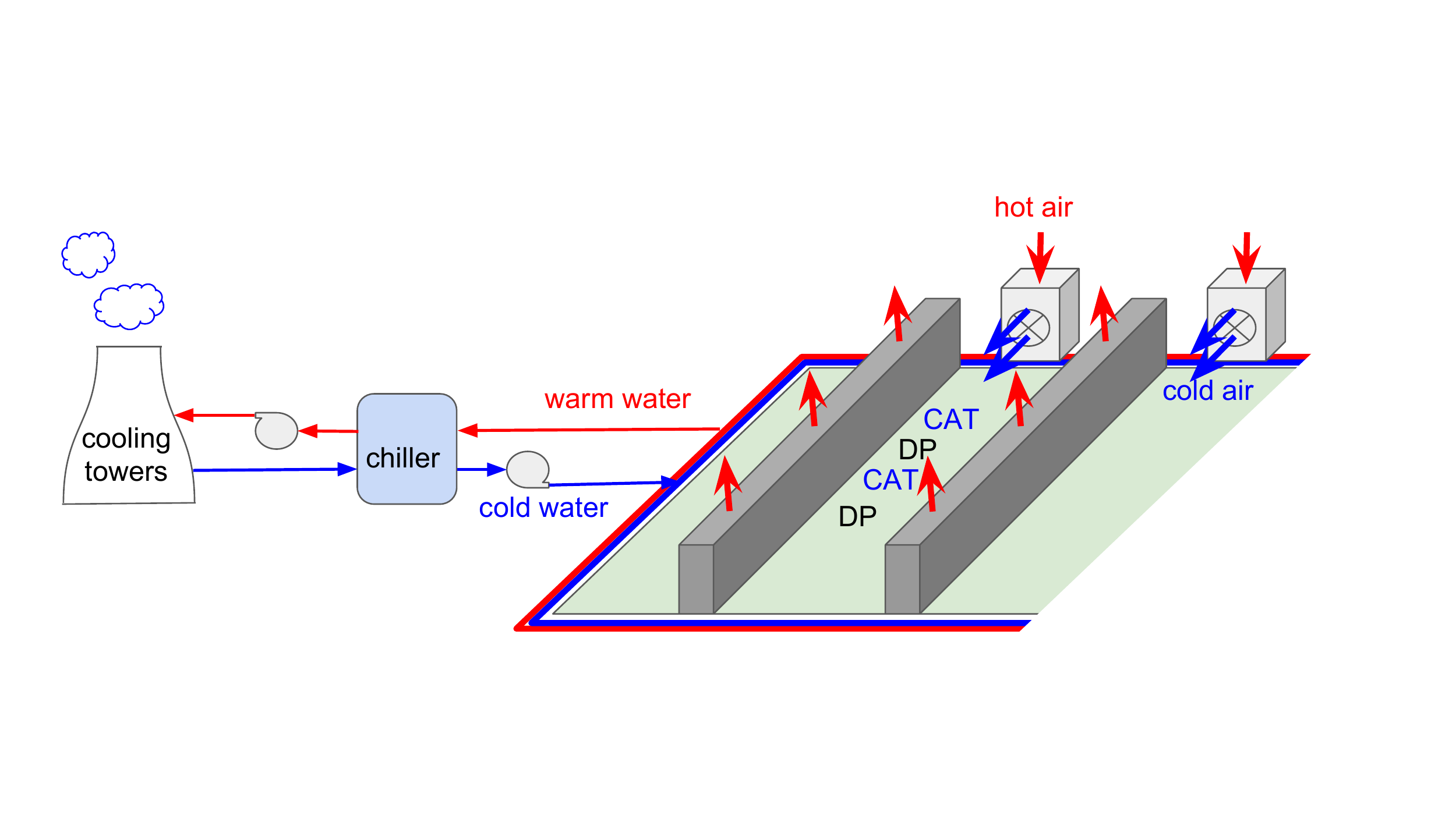}
\caption{Data center cooling loop; see \cref{sec:experiments}.}
\label{fig:dclayout}
\end{figure}

We demonstrate our approach on the problem of regulating conditions inside a data center (DC) server floor in the presence of time-varying power costs. We learn system dynamics from a real data center, but vary the costs and run algorithms in simulation.

\cref{fig:dclayout} shows a schematic of the cooling loop of a typical data center.  Water is cooled to sub-ambient temperatures in the chiller and evaporative cooling towers, and then sent to multiple air handling units (AHUs) on the server floor.
Server racks are arranged into rows with alternating hot and cold aisles, such that all hot air exhausts face the hot aisle.
The AHUs circulate air through the building; hot air is cooled through air-water heat exchange and blown into the cold aisle, and the resulting warm water is sent back to the chiller and cooling towers.
The primary goal of floor-level cooling is to control the cold aisle temperatures (CATs) and differential air pressures (DPs).
The control vector includes the blower speed and water valve command for each of $n=30$ AHUs, set every 30s. The state vector includes $2n$ temperature measurements and $n$ pressure measurements, as well as sensor measurements and controls for the preceding time step.  System noise is in part due to variability in server loads and the temperature of the chilled water.

We learn a linear approximation $(A, B)$ of the dynamics in the operating range of interest on 4h of exploratory data with controls following a random walk. We estimate the system noise covariance $W$ as the empirical covariance of training data residuals.  For the purpose of the experiment, we amplify the noise by a factor of 5.
We set the diagonal coefficients of $Q_t$ corresponding to the most recent (normalized) sensor measurements to 1 and remaining coefficients to 0, and keep $Q_t=Q$ constant throughout the experiment. We set diagonal coefficients of $R_t$ corresponding to water usage (valve command) to 1 throughout, and all coefficients corresponding to power usage (fan speed) to $r_t$. We generate $r_t$ by (a) i.i.d sampling a uniform distribution on $[0.1, 1]$, and (b) using a random walk restricted to $[0.1, 1]$ taking steps of size $0.1, -0.1, 0$ with probabilities $0.1, -0.1, 0.8$ respectively.

\begin{figure}[t]
\centering
\includegraphics[trim={1cm 0.5cm 1cm 1cm},clip,width=0.65\linewidth]{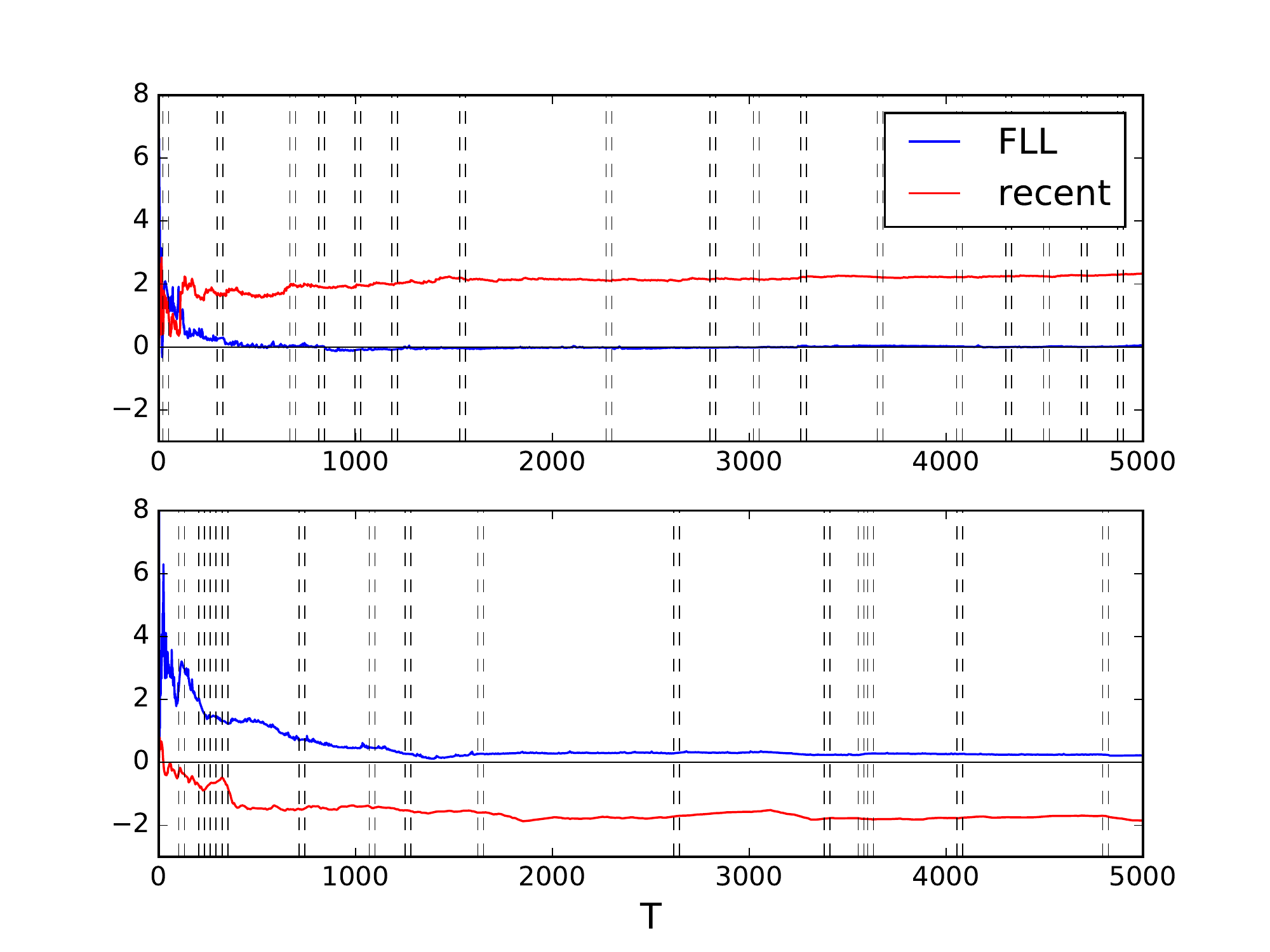}
\caption{Normalized regret $R_T/T$ for FLL and Recent strategies, with power costs generated uniformly (top) and by random walk (bottom). Resets occur at time steps indicated by dashed lines. }
\label{fig:dccool}
\end{figure}

We run the FLL algorithm on this problem with the following modifications: we set $Q_1^p=Q$, and $R_1^p = I_k$, an upper bound on $R_t$. Rather than executing hard resets to 0,
we perform a soft reset by running a policy $K_{reset}$ for $n$ steps.  Here $K_{reset}$ is similar to the next FLL policy, but based on the 1.1 times the corresponding state cost $Q$.

We compare the cost of FLL to that of a fixed linear controller that is based on the average of the $R_t$ matrices, and to a \emph{Recent} strategy which selects one of ten controllers corresponding to power costs in $r \in \{0.1, 0.2, ..., 1\}$ based on the most recently observed $R_t$. The  normalized regret  $\frac{1}{T}R_T$ of to the two strategies is shown in
\cref{fig:dccool}.  FLL performance quickly approaches that of the fixed linear policy in both cases, and is better than the \emph{Recent} strategy on uniform random costs. The \emph{Recent} strategy has an advantage in the case where costs vary slowly, and empirical performance of FLL could likely be improved in this case by forgetting the old costs.

\nocite{*}
\bibliographystyle{icml2018}
\bibliography{paper}

\appendix

\section{Technical Proofs}
\label{sec:proofs}

\subsection{Proof of \cref{lem:stab-trbound}}

\begin{proof}
Strong stability ensures that $\rho(A+BK) < 1$, and so
\begin{align*}
X
=
\sum_{t=0}^\infty (A+BK)^t W ((A+BK)^t)\tr
.
\end{align*}
Write $A+BK = HLH^{-1}$ such that $\norm{L} \le 1-\gamma$ and $\norm{H}\norm{H^{-1}} \le \kappa$.
Then
\begin{align*}
\norm{(A+BK)^t}
\le
\norm{H}\norm{H^{-1}} \norm{L}^t
\le
\kappa (1-\gamma)^t
.
\end{align*}
As a result,
\begin{align*}
\trace(X)
&\le
\sum_{t=0}^\infty \norm{(A+BK)^t}^2 \trace(W)
\\
&\le
\kappa^2 \sum_{t=0}^\infty (1-\gamma)^{2t} \trace(W)
\le
\frac{\kappa^2}{\gamma} \trace(W)
.
\end{align*}
Further, notice that $U = KXK\tr$, whence
\begin{align*}
\trace(U)
=
\trace(KXK\tr)
\le
\trace(X) \norm{K}^2
\le
\frac{\kappa^4}{\gamma} \trace(W)
.
&\qedhere
\end{align*}
\end{proof}

\subsection{Proof of \cref{lem:seqstab-mixing}}

\begin{proof}
Denote $C_t = \cov[u_t | x_t]$ (where $u_t$ is the action taken on round $t$).
By definition, for all $t$ we have
\begin{alignat*}{2}
&\wh{X}_{t+1}&
&=
(A+BK_t) \wh{X}_t (A+BK_t)\tr + B C_t B\tr + W
;
\\
&X_t&
&=
(A+BK_t) X_t (A+BK_t)\tr + B C_t B\tr + W
.
\end{alignat*}
Subtracting the equations, substituting $A+BK_t = H_t L_t H_t^{-1}$ and rearranging yields
\begin{align*}
H_t^{-1} (\wh{X}_{t+1} - X_{t}) (H_t^{-1})\tr
=
L_t H_t^{-1} (\wh{X}_t - X_{t}) (H_t^{-1})\tr L_t\tr
.
\end{align*}
Denote $\Delta_t = H_t^{-1} (\wh{X}_t - X_{t}) (H_t^{-1})\tr$ for all $t$.
Then the above can be rewritten as
\begin{align} \label{eq:DeltaLH}
\begin{aligned}
\Delta_{t+1}
&=
(H_{t+1}^{-1} H_t L_t) \Delta_t (H_{t+1}^{-1} H_t L_t)\tr
\\
&\quad
+ (H_{t+1}^{-1}) (X_t-X_{t+1}) (H_{t+1}^{-1})\tr
.
\end{aligned}
\end{align}
Let us first analyze the simpler case where all policies $K_t$ converge to the same steady-state covariance $X$.
Then $X_t = X$ for all~$t$, thus \cref{eq:DeltaLH} reads
\begin{align*}
\Delta_{t+1}
&=
(H_{t+1}^{-1} H_t L_t) \Delta_t (H_{t+1}^{-1} H_t L_t)\tr
.
\end{align*}
Taking norms, we obtain
\begin{align*}
\norm{\Delta_{t+1}}
&\le
\norm{L_t}^2 \norm{H_{t+1}^{-1} H_t}^2 \norm{\Delta_t}
\\
&\le
(1-\gamma)^2 (1+\thalf\gamma)^2 \norm{\Delta_t}
\\
&\le
(1-\thalf\gamma)^2 \norm{\Delta_t}
,
\end{align*}
whence $\norm{\Delta_{t+1}} \le e^{-\gamma t} \norm{\Delta_1}$.
Recalling $\wh{X}_t - X = H_t \Delta_t H_t\tr$, we obtain
\begin{align*}
\norm{\wh{X}_{t+1} - X}
&\le
e^{-\gamma t} \norm{\Delta_1} \norm{H_{t+1}}^2
\\
&\le
e^{-\gamma t} \norm{H_1^{-1}}^2 \norm{H_{t+1}}^2 \norm{\wh{X}_1 - X}
\\
&\le
\kappa^2 e^{-\gamma t} \norm{\wh{X}_1 - X}
.
\end{align*}
For the general case, taking norms in \cref{eq:DeltaLH} results in
\begin{align*}
\norm{\Delta_{t+1}}
&\le
(1-\thalf\gamma)^2 \norm{\Delta_t} + \norm{H_{t+1}^{-1}}^2 \norm{X_{t}-X_{t+1}}
\\
&\le
(1-\thalf\gamma)^2 \norm{\Delta_t} + \frac{\eta}{\alpha^2}
,
\end{align*}
and unfolding the recursion we obtain
\begin{align*}
\norm{\Delta_{t+1}}
&\le
(1-\thalf\gamma)^{2t} \norm{\Delta_1} + \frac{\eta}{\alpha^2} \sum_{s=0}^\infty (1-\thalf\gamma)^{2s}
\\
&\le
e^{-\gamma t} \norm{\Delta_1} + \frac{2\eta}{\alpha^2\gamma}
.
\end{align*}
Using $\wh{X}_t - X = H_t \Delta_t H_t\tr$ again and simple algebra give the result.\end{proof}

\subsection{Proof of \cref{lem:sdp-seqstab}}

\begin{proof}
Denote $X_t = (\Sigma_t)_{xx}$, and recall that, for all $t$,
\begin{align*}
\Sigma_t
\succeq
\Sigma_t'
=
\begin{pmatrix} X_t & K_t X_t \\ X_t K_t\tr & K_t X_t K_t\tr \end{pmatrix}
\end{align*}
(cf.~\cref{eq:SigSigpr}).
Now, since $\Sigma_t$ is feasible for the SDP we have
\begin{align*}
X_t
&=
\begin{pmatrix} A & B \end{pmatrix} \Sigma_t \begin{pmatrix} A & B \end{pmatrix}\tr + W
\\
&\succeq
\begin{pmatrix} A & B \end{pmatrix} \Sigma_t' \begin{pmatrix} A & B \end{pmatrix}\tr + W
\\
&\succeq
(A + BK_t) X_t (A+BK_t)\tr + \sigma^2 I
.
\end{align*}
Proceeding as in the proof of \cref{lem:sdp-sstable}, one can show that $\norm{K_t}_\frob \le \kappa$, and that the matrix $L_t = X_t^{-1/2}(A+BK_t) X_t^{1/2}$ satisfies $\norm{L_t} \le 1-1/2\kappa^2$ with $\norm{X_t^{1/2}} \le \sqrt{\trbound}$ and $\norm{X_t^{-1/2}} \le 1/\sigma$.
To establish sequential strong stability it thus suffices to show that $\norm{X_{t+1}^{-1/2}X_t^{1/2}} \le 1+1/4\kappa^2$ for all $t$.
To this end, observe that $\norm{X_{t+1}-X_t} \le \eta$,%
\footnote{\label{fn:blocknorm}We use the fact that for a symmetric
$
M = \Lr{\begin{smallmatrix} A & B \\ B\tr & D \end{smallmatrix}}
$
one has
$
\norm{M}
\ge
\max_{\norm{x} \le 1} \abs{(\begin{smallmatrix} x \\ 0 \end{smallmatrix})\tr M (\begin{smallmatrix} x \\ 0 \end{smallmatrix})}
=
\max_{\norm{x} \le 1} \abs{x\tr A x}
=
\norm{A}
.
$
}
and
\begin{align*}
\norm{&X_{t+1}^{-1/2}X_t^{1/2}}^2
\\
&=
\norm{X_{t+1}^{-1/2}X_t^{}X_{t+1}^{-1/2}}
\\
&\le
\norm{X_{t+1}^{-1/2}X_{t+1}^{}X_{t+1}^{-1/2}} + \norm{X_{t+1}^{-1/2}(X_{t+1}^{}-X_t^{})X_{t+1}^{-1/2}}
\\
&\le
1 + \norm{X_{t+1}^{-1/2}}^2 \norm{X_{t+1}^{}-X_t^{}}
\\
&\le
1 + \frac{\eta}{\sigma^2}
.
\end{align*}
Hence, if $\eta \le \sigma^2/\kappa^2$ then $\norm{X_{t+1}^{-1/2}X_t^{1/2}} \le \sqrt{1+1/\kappa^2} \le 1+1/2\kappa^2$ as required.
\end{proof}

\subsection{Proof of \cref{lem:ogd-regret}}

\begin{proof}
The diameter of the feasible domain of the SDP (with respect to $\norm{\cdot}_\frob$) is upper bounded by $2\trbound$.
Also, the linear loss function $X
\mapsto L_t \bullet X$ is $\norm{Q_t}_\frob + \norm{R_t}_\frob \le 2C$ Lipschitz for all $t$ (again, with respect to $\norm{\cdot}_\frob$).
Plugging this into the regret bound of the Online Gradient Descent algorithm gives the lemma.
\end{proof}

\subsection{Proof of \cref{lem:ogd-mixing}}

\begin{proof}
Denote the policy used by the algorithm on round $t$ by $\pi_t(x) = K_t x + v_t$ with $v_t \sim \calN(0,V_t)$.
Notice that
\begin{align*}
\Sigma_t
=
\SigK(K_t) + \begin{pmatrix} 0 & 0 \\ 0 & V_t \end{pmatrix}
,
\end{align*}
whence $\SigK(\pi_t) = \Sigma_t$.
Next, denote $X_t = (\Sigma_t)_{xx}$, $U_t = (\Sigma_t)_{uu}$, and similarly $\wh X_t = (\wh\Sigma_t)_{xx}$, $\wh U_t = (\wh\Sigma_t)_{uu}$.
Observe that $\wh U_t = K_t \wh X_t K_t\tr + V_t$ and $U_t = K_t X_t K_t\tr + V_t$, thus
\begin{align*}
L_t \bullet (\wh\Sigma_t - \Sigma_t)
&=
Q_t \bullet (\wh X_t - X_t) + R_t \bullet (\wh U_t - U_t)
\\
&=
(Q_t + K_t\tr R_t K_t) \bullet (\wh X_t - X_t)
\\
&\le
\trace(Q_t + K_t\tr R_t K_t) \, \norm{\wh X_t - X_t}
.
\end{align*}
Further, for $K = \KSig(\Sigma)$ for any feasible $\Sigma \in \cS$ we have
\begin{align} \label{eq:QKRKbound}
\trace(Q_t + K\tr R_t K)
\le
C(1+\bar\kappa^2)
\le
2C\bar\kappa^2
,
\end{align}
as $\trace(Q_t),\trace(R_t) \le C$, and $\norm{K}^2 \le \norm{K}_\frob^2 \le \bar\kappa^2$ thanks to \cref{lem:sdp-sstable}.
Hence,
\begin{align} \label{eq:LShSt}
\begin{aligned}
\sum_{t=1}^T L_t \bullet (\wh\Sigma_t - \Sigma_t)
\le
2C\bar\kappa^2 \sum_{t=1}^T \norm{\wh X_t - X_t}
.
\end{aligned}
\end{align}
It is left to control the norms $\norm{\wh X_t - X_t}$.
To this end, recall \cref{lem:sdp-seqstab} which asserts that the sequence $K_1,K_2,\ldots$ is $(\bar\kappa,\bar\gamma)$-strongly stable, since we assume $\eta \le \sigma^2/\bar\kappa^2$.
Now, since $\norm{X_{t+1} - X_t} \le \norm{\Sigma_{t+1} - \Sigma_t} \le 4C\eta$,
applying \cref{lem:seqstab-mixing} to the sequence of randomized policies $\pi_1,\pi_2,\ldots$ now yields
\begin{align} \label{eq:whXtXt}
\norm{\wh{X}_t - X_t}
\le
\bar\kappa^2 e^{-\bar\gamma t} \norm{\wh{X}_1 - X_1} + \frac{8C\eta\bar\kappa^2}{\bar\gamma}
.
\end{align}
We can further bound the right-hand side using $\norm{\wh{X}_1 - X_1} \le 2\trbound$.
Combining \cref{eq:LShSt,eq:whXtXt} and using the fact that
$\sum_{t=1}^T e^{-\alpha t} \le \int_{0}^\infty e^{-\alpha t} dt = 1/\alpha$ for $\alpha>0$, we obtain the result.\qed

\subsection{Proof of \cref{lem:Kst-mixing}}
Denote $X = \Sigma^\st_{xx}$ and $X_t = (\Sigma^\st_t)_{xx}$, and observe that
\begin{align*}
\Sigma^\st
=
\begin{pmatrix}
X & X (K^\star)\tr \\
K^\star X & K^\star X (K^\star)\tr
\end{pmatrix}
,~~
\wh\Sigma^\st_t
=
\begin{pmatrix}
X_t & X_t (K^\star)\tr \\
K^\star X_t & K^\star X_t (K^\star)\tr
\end{pmatrix}
.
\end{align*}
Thus $L_t \bullet (\Sigma^\st - \wh\Sigma^\st_t) = (Q_t + (K^\st)\tr R_t K^\st) \bullet (X - X_t)$.
Now, \cref{lem:stab-trbound} asserts that $\trace(\Sigma^\st) \le 2\kappa^4/\gamma = \trbound$, hence $\Sigma^\st \in \cS$ and by \cref{eq:QKRKbound}, it follows that
\begin{align*}
\sum_{t=1}^T L_t \bullet (\Sigma^\st - \wh\Sigma^\st_t)
\le
2C\kappa^2 \sum_{t=1}^T \norm{X - X_t}
~.
\end{align*}
Now, an application of \cref{lem:stab-mixing} gives
\begin{align*}
\sum_{t=1}^T \norm{X_t - X}
\le
\kappa^2 \norm{X_1 - X} \sum_{t=1}^T e^{-2\gamma t}
\le
\frac{\kappa^2}{2\gamma} \norm{X_1 - X}
~,
\end{align*}
where in the ultimate inequality we have used again the fact that $\sum_{t=1}^T e^{-\gamma t} \le 1/\gamma$.
Finally, we have
$
\norm{X_1 - X}
\le
\norm{\Sigma^\st - \wh\Sigma^\st_0}
\le
2\trbound
.
$
Combining the inequalities gives the result.
\end{proof}


\subsection{Proof of \cref{thm:ogd}}

\begin{proof}
Plugging in the bounds we established
in \cref{lem:ogd-mixing,lem:ogd-regret,lem:Kst-mixing}
into \cref{eq:regretdecomp1} and setting the values for $\bar\kappa=\sqrt{\trbound}/\sigma$ and $\bar\gamma=\sigma^2/2\trbound$ (and using $\trbound \ge \sigma^2$ to simplify), we obtain
\begin{align*}
\sum_{t=1}^T L_t \bullet (\wh\Sigma_t - \wh\Sigma^\st_t)
\le
\frac{40C^2\trbound^3}{\sigma^6} \eta T
+
\frac{4\trbound^2}{\eta}
+
\frac{8C\trbound^4}{\sigma^6} + \frac{2C \kappa^4 \nu}{\gamma}
\end{align*}
for any $\eta$ such that $\eta \le \sigma^2/4C\bar\kappa^2 = \sigma^4/4C\trbound$.
Thus, a choice of $\eta = \sigma^3/(2C\sqrt{\trbound T})$ (for which it can be verified that $\eta \le \sigma^4/4C\trbound$ for $T \ge 4\nu/\sigma^2$) gives the regret bound
\begin{align*}
\cost_T(A) - \cost_T(K^\st)
\le
20C \frac{\trbound^{2.5}}{\sigma^3} \sqrt{T}
+
\frac{8C\trbound^4}{\sigma^6} + \frac{2C \kappa^4 \nu}{\gamma}~.
\end{align*}
Finally, plugging in $\trbound = 2\kappa^4\lambda^2/\gamma$ gives the result.
\end{proof}

\subsection{Proof of \cref{thm:fplanalysis}}

\begin{proof}
Since after a reset, the system starts at state $0$, the cost of the learner is always less than the steady-state cost. The expected number of switches is at most $\eta C \sqrt{d+k} T$, and whenever a switch occurs we pay an additional cost of $C_r$ for performing a reset. Combining that with our bounds on the three terms in \cref{eq:regretdecomp2}, we get
\begin{align*}
\cost_T(A) - \cost_T(K^\st) &\le \eta C C_r \sqrt{d+k} T + 8 \eta C^2 \nu \sqrt{d+k} T \\
&\quad+ \frac{16 \nu (d+k)}{\eta} + \frac{2C \kappa^4 \nu}{\gamma}~.
\end{align*}
Setting $\eta = \sqrt{16 \nu \sqrt{d+k} / TC (C_r + 8C \nu)}$, and plugging in $\nu = 2 \kappa^4 \lambda^2/\gamma$ completes the theorem.
\end{proof}

\section{On Strong Stability}
\label{sec:generalB}

\newcommand{\conf}{k}
\newcommand{\cons}{\kappa}
\newcommand{\con}{(\conf, \cons)}

In this section, we give additional justification for the stability assumption. The following lemma shows that for any stable controller $K$, there are finite bounds on its strong stability parameters.

\begin{lemma}
\label{lem:lyapunov}
Suppose that for a linear system defined by $A, B$, a policy $K$ is stable. Then there are parameters $\kappa, \gamma>0$ for which it is $(\kappa, \gamma)$-strongly stable.
\end{lemma}
\begin{proof}
A theorem of Lyapunov says that a matrix $M$ is stable, i.e., its spectral radius is smaller than $1$ if and only if there exists a positive definite matrix $P$ such that
\begin{align*}
M\tr P M &\preceq P.
\end{align*}
Indeed $P= \sum_{i=0}^{\infty} (M^i)\tr (M^i)$ satisfies this condition.
Let $\rho(A+BK) = 1-\gamma$. Applying the above result to $(1-\gamma)^{-1}(A+BK)$, we conclude that for some positive definite matrix $P$,
\begin{align*}
(A+BK)\tr P (A+BK) \preceq (1- \gamma)^2 P.
\end{align*}
Pre- and post-multiplying by $P^{-\frac 1 2}$ and rearranging,
\begin{align*}
(P^{\frac 1 2} (A + BK) P^{-\frac 1 2})\tr(P^{\frac 1 2} (A + BK) P^{-\frac 1 2}) \preceq (1-\gamma)^2 I .
\end{align*}
Letting $Q=P^{\frac 1 2} (A + BK) P^{-\frac 1 2}$, we conclude that
$A+BK = P^{-\frac 1 2} Q P^{\frac 1 2}$ with $\|Q\| \leq 1-\gamma$. Letting $\kappa$ be the condition number of $P^{\frac 1 2}$, the claim follows.
\end{proof}

In the following sections, we give quantitative bounds on the strong stability parameters of optimal policies $K$, under certain more graspable assumptions on the system.

\subsection{Invertible $\bm{B}$}
As a warmup, we start with the setting when $B$ is invertible. We will show quantitative bounds on the trace bound $\trbound$ such that the optimal policy $K$ will be feasible for the SDP. A quantitative bound on the strong stability parameters will then follow from Lemma~\ref{lem:sdp-sstable}
\begin{lemma}
\label{lem:trace_invb}
Assume that $B$ is square and invertible, and $\trace(Q), \trace(R) \le C$ and $Q,R \succeq \mu I$.
Then for
$$\trbound = \frac{C \lambda^2}{\mu} (1 + \norm{B^{-1}A}^2),$$
the SDP is feasible and the trace constraint is not binding.
\end{lemma}

\begin{proof}
Consider the control matrix $K_0 = -B^{-1}A$, and let
\begin{align*}
\Sigma_0
=
\begin{pmatrix} W & W K_0\tr \\ K_0 W & K_0 W K_0\tr \end{pmatrix}
=
\begin{pmatrix} I & K_0\tr \end{pmatrix}\tr W \begin{pmatrix} I & K_0\tr \end{pmatrix}
.
\end{align*}
Then $\Sigma_0$ is PSD and, as $A+BK_0 = 0$, also satisfies
\begin{align*}
\begin{pmatrix} A & B \end{pmatrix} &~\Sigma_0  \begin{pmatrix} A & B \end{pmatrix}\tr + W
\\
&=
(A+BK_0) W (A+BK_0)\tr + W
=
(\Sigma_0)_{xx}
.
\end{align*}
%
Further, we have
\begin{align*}
\cost(\Sigma_0)
&=
(Q + K_0\tr R K_0) \bullet W
\\
&\le
\lambda^2 (\trace(Q) + \norm{K_0}^2 \trace(R))
\\
&\le
C \lambda^2 (1 + \norm{B^{-1}A}^2)
.
\end{align*}
On the other hand,
$
\cost(\Sigma_0)
=
\Lr{\begin{smallmatrix} Q & 0 \\ 0 & R \end{smallmatrix}} \bullet \Sigma_0
\ge
\mu \trace(\Sigma_0)
,
$
where we have used our assumption that $Q,R \succeq \mu I$.
Combining the two inequalities, we see that $\trace(\Sigma_0) \le \trbound$.
Thus, we proved that $\Sigma_0$ is feasible.

Finally, to see that the constraint $\trace(\Sigma) \le \trbound$ is not binding, consider the optimal solution $\Sigma^\st$ for the SDP excluding this constraint (which is, of course, also feasible).
Then, as before,
$
\cost(\Sigma_0)
\ge
\cost(\Sigma^\st)
=
\Lr{\begin{smallmatrix} Q & 0 \\ 0 & R \end{smallmatrix}} \bullet \Sigma^\st
\ge
\mu \trace(\Sigma^\st)
.
$
This again shows that $\trace(\Sigma^\st) \le \trbound$; that is, $\Sigma^\st$ satisfies the trace constraint, which is therefore not binding.
\end{proof}

\subsection{Controllability}
We now define general conditions on a linear system that allow us to prove quantitative bounds on the strong stability of the optimal solution.
We first recall the notion of {\em controllability} of a system. A system defined by $x_{t+1} = Ax_{t} + Bu_t$ is said to be {\em controllable} if the matrix $\begin{psmallmatrix}
B & AB & \cdots & A^{d-1} B
\end{psmallmatrix}$  is full rank. A standard result in control theory says that one can drive any state $x_0$ to zero if and only if the system is controllable.
We define a quantitative version of this condition.
\begin{definition}[$\con$-Strong Controllability]
A system defined by $x_{t+1} = Ax_t + Bu_t$ is $\con$-strongly controllable if the matrix $C_{\conf} = \begin{psmallmatrix} B & AB & \cdots & A^{\conf-1}B \end{psmallmatrix}$ satisfies $\|(C_{\conf}\tr C_{\conf})^{\dagger} \| \leq \cons$.
\end{definition}

We first show that for a strongly controllable system, any state $x_0$ can be driven to zero at bounded cost.
\begin{lemma}
\label{lem:cheap_reset}
Suppose that a dynamical system $x_{t+1} = Ax_t + Bu_t$ is $\con$-strongly controllable and that $Q$ and $R$ have spectral norm at most $1$. Let $a = \max(\|A\|, 1)$ and $b = \|B\|$. Then there is a constant $C = C(\conf, \cons, a, b)$ such that the system starting at a state $x_0$ can be driven to zero in $\conf$ steps at cost at most $C \|x_0\|^2$. I.e. there exist a  $x_1,\ldots, x_\conf, u_0,\ldots, u_{\conf-1}$ such that $x_\conf=0$, $x_{t+1} = Ax_t + Bu_t$ and
\begin{align*}
\sum_{t=0}^{\conf} x_t\tr Q x_t + \sum_{t=0}^{\conf-1} u_t\tr R u_t \leq C \|x_0\|^2
\end{align*}
\end{lemma}
\begin{proof}
Consider the following quadratic program:
\begin{align*}
&\min_{(u_t)_{t=0}^{\conf-1}} &&\sum_{t=0}^{\conf-1} \|u_t\|^2 \\
&\textrm{subject to} &&x_{t+1} = A x_t + B u_t,~t=0,\ldots,d-1 \\
& &&x_\conf = 0~.
\end{align*}
Rewriting, this is equivalent to
\begin{align*}
&\min_{(u_t)_{t=0}^{\conf-1}} &&\sum_{t=0}^{\conf-1} \|u_t\|^2 \\
&\textrm{subject to} &&C_{\conf}\tr \begin{psmallmatrix} u_{\conf -1} & u_{\conf - 2} & \cdots & u_0 \end{psmallmatrix} = - A^{\conf}x_{0}.
\end{align*}
By lemma~\ref{lemma:qpsolquad}, the optimal solution is given by $(C_{\conf}\tr C_{\conf})^{\dagger} A^{\conf} x_0$, so that
\begin{align*}
\sum_{t=0}^{\conf - 1} \|u_t\|^2 = \|(C_{\conf}\tr C_{\conf})^{\dagger} A^{\conf} x_0\|^2
\leq \cons^2 a^{2\conf} \|x_0\|^2~.
\end{align*}
For this  setting of $u_t$'s, the corresponding $x_t$'s satisfy
\begin{align*}
\|x_t\|
&= \|A^{t} x_0 + \sum_{i=0}^{t-1} A^{t-i-1}Bu_i\|
\\
&\le \|A\|^t \|x_0\| + \sum_{i=0}^{t-1} \|A\|^{t-i-1} \|B\| \|u_i\|~.
\end{align*}
An easy calculation then shows that for this solution,
\begin{align*}
\sum_{t=0}^{\conf} x_t\tr Q x_t &+ \sum_{t=0}^{\conf-1} u_t\tr R u_t
\\
&\leq
\sum_{t=0}^{\conf} \|x_t\|^2 + \sum_{t=0}^{\conf-1} \|u_t\|^2\\
&\leq
\sum_{t=0}^{\conf} k(a^{2t} \|x_0\|_2^2 + \sum_{i=1}^{t-1} a^{2(t-i-1)}b^2\|u_i\|_2^2) \\&\qquad+ \sum_{t=0}^{\conf-1} \|u_t\|^2\\
&\leq \conf^2 a^{2\conf} \|x_0\|^2 + (1+\conf^2 (a^{2\conf2} b^2)) \sum_{t=0}^{\conf-1} \|u_t\|^2\\
&\leq (\conf^2 a^{2\conf} + \cons^2 a^{2\conf}(1+ \conf^2 a^{2\conf} b^2)) \|x_0\|^2~.
\qedhere
\end{align*}
\end{proof}

We now prove a generalization of Lemma~\ref{lem:trace_invb} in terms of the zeroing cost.

\begin{theorem}
[Trace Bound]
\label{thm:trbound}
Suppose that matrices $A$ and $B$ are such that for any $x_0$, the system $x_{t+1} = Ax_t + Bu_t$ can be driven to zero in $\conf$ steps at cost $C \|x_0\|^2$ for cost matrices $Q=I, R=I$. Consider the noisy system $x_{t+1} = Ax_t + B u_t + w_t$ with $w_t \sim N(0, W)$. Then for
$\trbound = C \cdot \trace(W)$,
the SDP is feasible and the trace constraint is not binding.
\end{theorem}

\begin{proof}
By assumption, given $x_0 = x$, there is a sequence of actions $u_0(x), u_1(x),\ldots u_{\conf-1}(x)$ and corresponding states $x=x_0(x),x_1(x),\ldots, x_{\conf}(x)=0$ such that $\sum_{t=0}^{\conf-1} (\|x_i(x)\|^2 + \|u_i(x)\|^2) \leq C \|x\|^2$.
Consider the covariance matrices
\begin{align*}
\Sigma_{xx}^{(t)} &\eqdef \E_{w \sim N(0,W)}[x_t(w) x_t(w)\tr]~,\\
\Sigma_{uu}^{(t)} &\eqdef \E_{w \sim N(0,W)}[u_t(w) u_t(w)\tr]~,\\
\Sigma_{xu}^{(t)} &\eqdef \E_{w \sim N(0,W)}[x_t(w) u_t(w)\tr]~.
\end{align*}
From the fact that $x_{t+1} = Ax_t + Bu_t$, it follows that
\begin{align*}
\Sigma_{xx}^{(t+1)} &= \E[(Ax_t + Bu_t)(Ax_t + Bu_t)\tr]\\
&= A\Sigma_{xx}^{(t)} A\tr + B\Sigma_{uu}^{(t)} B\tr + A\Sigma_{xu}^{(t)} B\tr + B(\Sigma_{xu}^{(t)})\tr A\tr \\
&= \begin{pmatrix}
A & B
\end{pmatrix} \begin{pmatrix}
\Sigma_{xx}^{(t)} & \Sigma_{xu}^{(t)}\\
(\Sigma_{xu}^{(t)})\tr & \Sigma_{uu}^{(t)}\\
\end{pmatrix} \begin{pmatrix}
A & B
\end{pmatrix}\tr
\end{align*}
Moreover, $\Sigma_{xx}^{(0)} = W$ and $\Sigma_{xx}^{(\conf)} = 0$.
Now consider the matrices:
\begin{align*}
\Sigma_{xx} \eqdef \sum_{t=0}^{\conf-1} \Sigma_{xx}^{(t)},\quad
\Sigma_{uu} \eqdef \sum_{t=0}^{\conf-1} \Sigma_{uu}^{(t)},\quad
\Sigma_{xu} \eqdef \sum_{t=0}^{\conf-1} \Sigma_{xu}^{(t)}.
\end{align*}
We claim that the matrix $\Sigma = \begin{pmatrix}
\Sigma_{xx} & \Sigma_{xu}\\
\Sigma_{xu}\tr & \Sigma_{uu}\\
\end{pmatrix}$ satisfies the SDP equality constraint. Indeed
\begin{align*}
\Sigma_{xx} &= \sum_{t=0}^{\conf-1} \Sigma_{xx}^{(t)}\\
&= \Sigma_{xx}^{(0)} + \sum_{t=1}^{\conf-1} \Sigma_{xx}^{(t)}\\
&= \Sigma_{xx}^{(0)} + \sum_{t=1}^{\conf} \Sigma_{xx}^{(t)}\\
&= \Sigma_{xx}^{(0)} + \sum_{t=0}^{\conf-1} \Sigma_{xx}^{(t+1)}\\
&= W + \sum_{t=0}^{\conf-1} \begin{pmatrix}
A & B
\end{pmatrix} \begin{pmatrix}
\Sigma_{xx}^{(t)} & \Sigma_{xu}^{(t)}\\
(\Sigma_{xu}^{(t)})\tr & \Sigma_{uu}^{(t)}\\
\end{pmatrix} \begin{pmatrix}
A & B
\end{pmatrix}\tr\\
&= W + \begin{pmatrix}
A & B
\end{pmatrix} \left( \sum_{t=0}^{\conf-1} \begin{pmatrix}
\Sigma_{xx}^{(t)} & \Sigma_{xu}^{(t)}\\
(\Sigma_{xu}^{(t)})\tr & \Sigma_{uu}^{(t)}\\
\end{pmatrix} \right) \begin{pmatrix}
A & B
\end{pmatrix}\tr\\
&= W + \begin{pmatrix}
A & B
\end{pmatrix}
\Sigma \begin{pmatrix}
A & B
\end{pmatrix}\tr
\end{align*}
Further $\Sigma \succeq 0$ and we have
\begin{align*}
Tr(\Sigma)
&= Tr(\Sigma_{xx}) + Tr(\Sigma_{uu})\\
&= \sum_{t=0}^{\conf-1} \left(Tr(\Sigma_{xx}^{(t)}) + Tr(\Sigma_{uu}^{(t)})\right)\\
&=  \sum_{t=0}^{\conf-1} \left( \E_{w \sim N(0,W)} (\|x_t(w)\|^2 + \|u_t(w)\|^2) \right)\\
&= \E_{w \sim N(0,W)} \sum_{t=1}^{\conf-1} (\|x_t(w)\|^2 + \|u_t(w)\|^2)\\
&\le  \E_{w \sim N(0,W)}  C \|w\|^2\\
&= C \cdot \trace(W).
\qedhere
\end{align*}
\end{proof}

\subsection{Solving Least Squares}
The following is a standard fact about least squares regression; we give a proof for completeness.
\begin{lemma}
\label{lemma:qpsolquad}
Consider a QP: $\min_x x\tr A x$ subject to $Bx = c$, where $A$ is PD.
Then the value of the optimal solution is $c\tr (B A^{-1} B\tr)^\dagger c$.
\end{lemma}

\begin{proof}
When minimizing any convex function over the constraint $Bx = c$, the gradient at the optimal solution is in the null space of $B$. Namely, there is some $\lambda$ such that $A x = B\tr \lambda$. Combining that with the constraint $Bx=c$, we can choose $\lambda = (B A^{-1} B\tr)^\dagger c$. Setting that into the objective function, we get the desired result.
\end{proof}

\section{Bounding the Reset Cost}
\label{sec:reset_cost}

Here we argue that under reasonable assumptions, the reset cost can be bounded. It will be useful to have some bound on the cost of driving a state to zero for a noiseless system. Lemma~\ref{lem:cheap_reset} gives such a bound under the Strong Controllability assumption. We next give an alternate bound coming from the existence of a strongly stable policy.

\begin{lemma}[Zeroing using Strong Stability]
\label{lem:stability_to_reset}
Suppose that a linear system has a $(\kappa, \gamma)$-strongly stable policy $K$. Then for $Q, R \preceq I$, a start state $x_0$ can be driven to norm at most $\frac 1 {T^2}$ in $O(\log (T\|x_0\|) / \gamma)$ steps at cost
\begin{align*}
 \frac{d(1+\kappa)\kappa^2 \|x_0\|_2^2}{2\gamma} .
\end{align*}
\end{lemma}
\begin{proof}
Let $K$ be a $(\kappa, \gamma)$-strongly stable policy. We argue that playing $K$ for approximately $t_{mix}=(1/\gamma) \log (\|x_0\|T^2)$ steps nearly zeroes the state; indeed at that point, the residual norm falls to below $\frac{1}{CT^2}$ so that the overall overhead coming from this residual is $o(1)$.  This is a consequence of the fact that for the noiseless model, the steady state $X$ is zero. By Lemma~\ref{lem:stab-mixing}
\begin{align*}
\|X_t\| &\leq \kappa^2 \exp(-2\gamma t) \|X_0\|,
\end{align*}
so that
\begin{align*}
\|x_t\|_2^2 &=  Tr(X_t)\\
&\leq d\|X_t\|\\
&\leq d\kappa^2 \exp(-2\gamma t) \|X_0\|\\
&\leq d\kappa^2 \exp(-2\gamma t) \|x_0\|_2^2.
\end{align*}
Moreover the cost of this near-reset is at most
\begin{align*}
(1+\kappa) \sum_{t=1}^{t_{mix}} \|x_t\|^2 &\leq d(1+\kappa)\kappa^2 \sum_{t=1}^{t_{mix}} \exp(-2\gamma t) \|x_0\|^2\\
&\leq \frac{d(1+\kappa)\kappa^2 \|x_0\|_2^2}{2\gamma}.
\qedhere
\end{align*}
\end{proof}

The following Theroem shows how resetting can be done in the general case using Lemma~\ref{lem:cheap_reset}, or Lemma~\ref{lem:stability_to_reset} to bound the cost of driving a state to zero. It follows that under either the assumption of strong controllability, or the existence of a strongly stable policy, we can derive a bound on the cost $C_r$ of the reset step in FLL.

\begin{theorem}[Resetting]
\label{thm:generalreset}
Consider the noisy system $x_{t+1} = Ax_t + Bu_t + w_t$ where $w_t \sim N(0, W)$. Suppose that
\begin{itemize}
\item The noiseless system $x_{t+1} = Ax_t + Bu_t$ starting from $x_0$ can be driven to zero in $\conf$ steps at cost $C \|x_0\|^2$, and
\item for some strategy $K$,  the noisy system starting at state $0$ has steady state cost $C_{ss}$ and steady state covariance $\Sigma_{xx}$.
\end{itemize}
Then, given initial state $x_0$, there is a sequence of actions $u_0,\ldots, u_{\conf-1}$ such that state $\E[x_{\conf}x_{\conf}\tr] \preceq \Sigma_{xx}$ and the cost of the the first $\conf$ steps is at most $\conf C_{ss} + C \|x_0\|^2$.
\end{theorem}
\begin{proof}
The idea is to use the linearity of the transition function, and split the sequence of $\conf$ states into two sequences: one that starts at $x_0$ and the other at $y_0 =0$. We will play $K$ on the sequence $y_0,\cdots,y_{\conf-1}$, and simultaneously drive the sequence $x_0,\cdots,x_{\conf-1}$ to 0.
Let $u_0, \ldots, u_{\conf-1}$ be the set of actions that drive the noiseless system starting at $x_0$ to zero.
At time $t$, we will play the control vector $u_t + K y_t$ where the actual state of the system is $x_t + y_t$. Thus we obtain $x_{t+1} = A x_t + B u_t$ and $y_{t+1} = (A+BK) y_t + w_t$, and indeed
\[
x_{t+1}+y_{t+1} = A (x_t + y_t) + B (u_t + K y_t) + w_t~.
\]
After $\conf$ rounds we will have $x_\conf= 0$, and the system would be at state $y_\conf$. A simple induction proof along the lines of \cref{lemma:ssafterreset} implies that $\E[x_{\conf} x_{\conf}\tr] \preceq \Sigma_{ss}$.
Finally, the sequences $x_0,\ldots,x_{d-1}$ and $y_0,\ldots,y_{d-1}$ are statistically-independent and $(y_t)_{t=0}^{d-1}$ has mean-zero. As the cost is a quadratic function of the state, the total expected cost of the reset is the sum of the expected costs of the two sequences individually.
\end{proof}

\end{document}